\documentclass{article} 
\usepackage{iclr2027_conference,times}

\usepackage{amsmath,amsfonts,bm}

\def\eqref#1{equation~\ref{#1}}

\def\1{\bm{1}}

\DeclareMathAlphabet{\mathsfit}{\encodingdefault}{\sfdefault}{m}{sl}
\SetMathAlphabet{\mathsfit}{bold}{\encodingdefault}{\sfdefault}{bx}{n}

\usepackage{multirow}
\usepackage{hyperref}
\usepackage{url}
\usepackage{booktabs}
\usepackage[table]{xcolor}
\usepackage{threeparttable}
\usepackage{adjustbox}
\usepackage{xcolor}
\usepackage[most]{tcolorbox}
\usepackage{amssymb}
\usepackage{algorithm}
\usepackage{algpseudocode}
\tcbuselibrary{listings,breakable,skins}

\usepackage[utf8]{inputenc} 
\usepackage[T1]{fontenc}    

\usepackage{amsfonts}       
\usepackage{nicefrac}       
\usepackage{microtype}      
\usepackage{xcolor}         
\usepackage{amsmath,amssymb}
\usepackage{amsthm}
\newtheorem{assumption}{Assumption}
\newtheorem{lemma}{Lemma}
\newtheorem{theorem}{Theorem}
\newtheorem*{remark*}{Remark}
\newtheorem{corollary}{Corollary}
\theoremstyle{remark}

\theoremstyle{plain}
\theoremstyle{definition}

\usepackage{graphicx}
\theoremstyle{plain}
\newtheorem{proposition}{Proposition}

\usepackage[most]{tcolorbox}

\definecolor{promptframe}{HTML}{3B6EA8}
\definecolor{promptback}{HTML}{F5F8FC}

\definecolor{takeawayframe}{HTML}{8470FF}

\newtcolorbox{takeaway}[1][]{
    enhanced, breakable,
    colback=takeawayframe!6,
    colframe=black,
    coltext=black,
    boxrule=0.6pt, left=6pt, right=6pt, top=4pt, bottom=4pt,
    fonttitle=\bfseries\small,
    coltitle=takeawayframe!6,
    title={Takeaway}, #1
}

\definecolor{citepink}{HTML}{8470FF}

\hypersetup{
    colorlinks=true,
    citecolor=citepink,  
    linkcolor=red,       
    urlcolor=black,
    filecolor=black
}

\title{Scaling Laws for Collapse in Asynchronous GRPO}

\author{
Jingwei Song\textsuperscript{1,2,3}\thanks{Equal contribution. Email: songjingwei@connect.hku.hk},
Haofeng Xu\textsuperscript{1}\footnotemark[1],
Jie Xiao\textsuperscript{3},
Chengke Bao\textsuperscript{3},
Jingwei Shi\textsuperscript{5},
Pengbin Feng\textsuperscript{4},\\
\textbf{Yuhang Han\textsuperscript{2}},
\textbf{Weixun Wang\textsuperscript{6}},
\textbf{Eric Yang\textsuperscript{3}},
\textbf{Tianyu Shi\textsuperscript{7}}\thanks{Corresponding author. Email: tianyu.shi3@mcgill.ca}
\\[2mm]
\textsuperscript{1}The University of Hong Kong
\textsuperscript{2}Shanghai Jiao Tong University
\textsuperscript{3}Gradient\\
\textsuperscript{4}University of Southern California
\textsuperscript{5}The Hong Kong Polytechnic University
\textsuperscript{6}Alibaba Group
\\
\textsuperscript{7}McGill University
}

 \newtcblisting{promptbox}[1]{%
    breakable,
    enhanced,
    arc=2pt,
    boxrule=0.5pt,
    colframe=promptframe,
    colback=promptback,
    coltitle=white,
    fonttitle=\bfseries\small,
    title={#1},
    top=4pt,
    bottom=4pt,
    left=6pt,
    right=6pt,
    listing only,
    listing options={
      basicstyle=\ttfamily\small,
      breaklines=true,
      breakatwhitespace=true,
      keepspaces=true,
      columns=fullflexible,
      breakindent=0pt,
      postbreak={}
    }
  }

\iclrfinalcopy 
\begin{document}

\maketitle
\lhead{Preprint.} 
\begin{abstract}
Asynchronous reinforcement learning improves the throughput of large language model post-training by decoupling rollout generation from policy optimization, but introduces a mismatch between the behavior and learner policies.
How the resulting policy staleness couples with the learning rate to govern training stability and collapse time remains poorly understood.
We investigate this coupling in vanilla GRPO through controlled sweeps of the synchronization interval $S$ and constant learning rate $\eta$ on Llama-3.2-1B/3B, complemented by experiments on Qwen3-8B.
We identify two empirical scaling laws:
\emph{(i) Stability-boundary scaling:} the largest stable learning rate scales approximately as $S^{-1}$, yielding a stability boundary characterized by an approximately constant product $S\eta$.
\emph{(ii) Collapse-time scaling:} among collapsing runs, estimated collapse times scale approximately as $\eta^{-1}$, corresponding to a model- and setup-dependent cumulative learning-rate budget that aligns across synchronization intervals in the Llama sweeps.
We interpret these laws through a local analysis of the behavior-dependent GRPO surrogate and a complementary mean-field model.
Under local regularity conditions, the analysis yields an $O(S\eta)$ upper bound on the staleness-induced update bias that resets at synchronization.
The mean-field model shows how sufficiently strong positive feedback can sustain directional drift when update directions persist across synchronization cycles.
When drift speed saturates under optimizer normalization, this mechanism predicts exit from a local surrogate-validity region after an approximately fixed cumulative learning rate.
Together, these findings motivate a practical calibration rule: estimate the stability threshold and collapse budget from a coarse sweep, then jointly select $S$ and $\eta$ for the intended training horizon.


\end{abstract}

\section{Introduction}
\label{sec:introduction}

Reinforcement learning (RL) is a central component of large language model (LLM) post-training~\citep{ouyang2022training,lambert2024tulu}, with verifiable rewards enabling substantial advances in reasoning~\citep{guo2025deepseek,kimi2025k15}.
Long, variable-length rollouts create a systems bottleneck: synchronous pipelines must wait for the slowest response in each batch, leaving computational resources idle.
Asynchronous systems alleviate this bottleneck by decoupling rollout generation from policy optimization, allowing generation and learner updates to proceed concurrently~\citep{noukhovitch2025asynchronous,fu2025areal,wu2025llamarl,piche2025pipelinerl}.
The resulting throughput gain comes with \emph{policy staleness}: the learner updates on trajectories generated by an earlier behavior policy.
This off-policy mismatch creates a tension between system efficiency and optimization stability.

Existing approaches mitigate staleness by bounding policy lag~\citep{fu2025areal,khatri2025art} or modifying policy updates to accommodate stale trajectories~\citep{li2025a3po,zheng2025prosperity,tian2026mugrpo,xu2026gac}.
Complementary analyses establish sufficient conditions for policy improvement when the behavior and learner policies remain close~\citep{mroueh2026revisiting,yao2026grouprelative}.
However, lag measured in update steps does not determine policy mismatch on its own: the displacement accumulated during those steps also depends on the learning rate.
A lag limit chosen in isolation therefore cannot characterize either the stability of a run or how long an unstable run may improve before collapsing.
We ask: \emph{How do staleness and learning rate jointly determine the stability boundary and the time to collapse beyond it?}

We investigate this question in vanilla Group Relative Policy Optimization (GRPO)~\citep{shao2024deepseekmath}, without additional staleness corrections.
Under periodic weight synchronization, we vary the synchronization interval $S$, measured in learner updates, and a constant learning rate $\eta$.
Controlled sweeps on Llama-3.2-1B/3B reveal two empirical scaling laws:
\textbf{(i) Stability-boundary scaling.}
Within the evaluated training horizons, the stability boundary is approximately characterized by a constant product $S\eta$.
The largest stable learning rate therefore scales approximately as $S^{-1}$: larger synchronization intervals require proportionally smaller learning rates.
\textbf{(ii) Collapse-time scaling.}
Among collapsing runs, the estimated onset time $t_{\mathrm{obs}}$, measured in learner updates, scales approximately as $\eta^{-1}$.
Equivalently, $t_{\mathrm{obs}}\eta\approx c_2$, where the cumulative learning-rate budget $c_2$ is model- and setup-dependent but approximately invariant to $S$ within each Llama sweep.
Complementary Qwen3-8B experiments at $S=16$ support this timing relationship with a different fitted budget.
These laws distinguish the stability of a configuration from its failure time.
Reducing the learning rate can move training below the stability boundary, but it can also merely delay collapse while leaving the configuration unstable.
A longer period of improving reward is therefore not, by itself, evidence that instability has been resolved.

To interpret this separation, we analyze the GRPO update field while explicitly retaining its dependence on the behavior policy.
Unlike delayed-gradient analyses for a fixed objective~\citep{lian2015asynchronous,zhang2016staleness,mishchenko2022asynchronous}, this formulation accounts for changes in both the sampling distribution and the surrogate weighting.
Under local regularity conditions, we derive an $O(S\eta)$ upper bound on the staleness-induced update bias, with a sawtooth envelope that resets at synchronization.
A complementary mean-field model connects this local discrepancy to persistent drift: when the bias reinforces recent displacement and update directions persist across synchronization cycles, positive feedback can sustain coherent motion above a threshold controlled by learning rate and average lag.
If directional coherence and the characteristic normalized step size saturate~\citep{kingma2015adam,balles2018dissecting,kunstner2023noise}, the learner advances at an approximately constant speed per unit of cumulative learning rate.
Modeling collapse as exit from a local region where the surrogate remains reliable then yields an approximately fixed budget, provided the effective exit geometry remains comparable across synchronization intervals.

These findings motivate joint calibration of staleness and learning rate for a fixed model and training setup.
A coarse sweep can bracket the stability boundary and estimate the collapse budget, guiding the selection of $(S,\eta)$ while distinguishing reduced instability from delayed failure.
Selected configurations should then be validated over the intended training horizon.

\textbf{Contributions.}
{\setlength{\leftmargini}{1.4em}
\begin{enumerate}
    \setlength{\itemsep}{1pt}
    \setlength{\topsep}{2pt}
    \item \textbf{Two empirical scaling laws.}
    We characterize an approximately $S\eta$-controlled stability boundary and an inverse-learning-rate collapse clock through controlled GRPO sweeps (Section~\ref{sec:experiments}).
    \item \textbf{Local dynamical analysis.}
    We derive a local bias bound and develop a conditional model connecting staleness feedback, persistent drift, and optimizer normalization to the two scaling relationships (Section~\ref{sec:method}).
    \item \textbf{Joint calibration.}
    We formulate a rule for selecting synchronization intervals and learning rates using an empirical stability bracket and cumulative learning-rate budget (Section~\ref{subsec:practical}).
\end{enumerate}}

\section{Preliminaries}
\label{sec:preliminaries}

\subsection{Group Relative Policy Optimization}
\label{subsec:prelim-grpo}

Let $\pi_\theta$ be an autoregressive policy with parameters
$\theta\in\mathbb R^d$, initialized at $\theta_0$.
For a prompt $x\sim\mathcal D$, GRPO~\citep{shao2024deepseekmath}
samples $G$ responses $y_{1:G}\sim\pi_\phi(\cdot\mid x)$ from a
behavior policy $\pi_\phi$, assigns binary verifiable rewards
$r_i=r(x,y_i)$, and computes the group-normalized advantage
\begin{equation}
\hat A_i
=
\frac{r_i-\operatorname{mean}(r_{1:G})}
{\operatorname{std}(r_{1:G})+\epsilon_A},
\label{eq:grpo-advantage}
\end{equation}
where $\epsilon_A>0$ stabilizes the normalization.
With token-level importance ratios
$\rho_{i,l}(\theta,\phi)=
\pi_\theta(y_{i,l}\mid x,y_{i,<l})/
\pi_\phi(y_{i,l}\mid x,y_{i,<l})$,
the standard GRPO surrogate for a group instance
$z=(x,y_{1:G})$ is
\begin{equation}
\ell(z;\theta,\phi)
=
\frac{1}{G}\sum_{i=1}^{G}\frac{1}{|y_i|}
\sum_{l=1}^{|y_i|}
\Big[
\min\!\big(
\rho_{i,l}\hat A_i,\,
\operatorname{clip}(\rho_{i,l},1-\epsilon,1+\epsilon)\hat A_i
\big)
-\beta\widehat D_{\mathrm{KL}}^{\,i,l}(\theta)
\Big],
\label{eq:grpo-surrogate}
\end{equation}
where $\epsilon$ is the clipping parameter and $\beta$ weights
the sampled KL penalty relative to the fixed reference policy
$\pi_{\mathrm{ref}}=\pi_{\theta_0}$.

Writing $p_\phi$ for the distribution of group instances induced
by $\mathcal D$ and $\pi_\phi$, we define the
\emph{surrogate-gradient field} and its on-policy counterpart as
\begin{equation}
\mathcal H(\theta,\phi)
\triangleq
\mathbb E_{z\sim p_\phi}
\big[\nabla_\theta\ell(z;\theta,\phi)\big],
\qquad
\mathcal H_{\mathrm{on}}(\theta)
\triangleq\mathcal H(\theta,\theta).
\label{eq:H-def}
\end{equation}
Differentiation holds the behavior policy and sampled advantages
fixed; the diagonal field $\mathcal H_{\mathrm{on}}$ need not
be the gradient of a fixed scalar objective.

Let $\theta_t$ denote the parameters after $t$ learner updates.
A mini-batch $\mathcal B_t$ sampled from $p_{\phi_t}$ gives
the gradient estimate and parameter update
\begin{equation}
g_t
\triangleq
\frac{1}{|\mathcal B_t|}
\sum_{z\in\mathcal B_t}
\nabla_\theta\ell(z;\theta_t,\phi_t),
\qquad
\theta_{t+1}=\theta_t+\eta_tu_t,
\label{eq:grpo-gradient}
\end{equation}
where $u_t$ is the optimizer update in ascent notation
(AdamW in our experiments).
With $\mathcal F_t$ denoting the history before sampling
$\mathcal B_t$, the sampling model gives
$\mathbb E[g_t\mid\mathcal F_t]=\mathcal H(\theta_t,\phi_t)$.
The idealized on-policy reference sets $\phi_t=\theta_t$,
recovering $\mathcal H_{\mathrm{on}}(\theta_t)$.
We write $\Lambda_t\triangleq\sum_{s<t}\eta_s$ for the
\emph{cumulative learning rate}; our experiments use a constant
rate $\eta$, so $\Lambda_t=t\eta$.

\subsection{Staleness and Collapse}
\label{subsec:prelim-async}

\paragraph{Stale rollouts.}
Asynchronous pipelines decouple rollout generation from policy optimization.
We consider periodic weight synchronization every $S$ learner
updates, holding the behavior policy fixed between synchronizations.
Abstracting away additional queueing delays, the behavior
parameters at step $t$ are
\begin{equation}
\phi_t\triangleq\theta_{t-k_t},
\qquad
k_t\triangleq t-S\lfloor t/S\rfloor\in\{0,\ldots,S-1\}.
\label{eq:phi-t}
\end{equation}
Here $S$ is the \emph{synchronization interval}, while the
\emph{policy lag} $k_t$ grows within each cycle, resets at
synchronization, and has cycle average $\bar k=(S-1)/2$.
The case $S=1$ recovers the idealized on-policy reference.
The resulting parameter discrepancy accumulates the intervening
optimizer updates:
\begin{equation}
\theta_t-\phi_t
=
\sum_{j=t-k_t}^{t-1}\eta_j u_j.
\label{eq:prelim-policy-displacement}
\end{equation}
Thus, lag counts updates, whereas the parameter discrepancy also
depends on their learning rates and directions.
The deviation
$\mathcal H(\theta_t,\phi_t)-\mathcal H_{\mathrm{on}}(\theta_t)$
is the \emph{staleness bias} analyzed in Section~\ref{sec:method}.

\paragraph{Collapse.}
We define \emph{training collapse} as an abrupt, severe drop in
mean verifier reward on training rollouts, with validation accuracy
providing corroborating evidence.
Under a fixed prompt distribution and decoding configuration,
each behavior-policy snapshot has a fixed expected reward,
although empirical batch rewards fluctuate.
We report a \emph{policy-indexed collapse time}
in learner-update units:
\begin{equation}
t_{\mathrm{obs}}\triangleq S M_{\mathrm{collapse}},
\label{eq:collapse-timestamp}
\end{equation}
where $M_{\mathrm{collapse}}$ indexes the first synchronized
behavior-policy snapshot exhibiting this drop, with initialization
indexed by zero.
A run is \emph{stable} if no collapse is observed within
its training horizon.

\section{Understanding the Scaling Laws}
\label{sec:method}

We interpret \emph{stability-boundary scaling} and
\emph{collapse-time scaling} through the within-cycle policy
discrepancy, the coherence of the resulting drift, and the distance
traveled per unit of cumulative learning rate.
The analysis separates a local property of the GRPO field from a
conditional dynamical model: regularity controls the magnitude of stale
feedback, while its sign, persistence, and normalized speed determine
whether that feedback can produce a collapse clock.
All norms below are per-coordinate RMS norms,
$\|v\|=\|v\|_2/\sqrt d$, with inner product
$\langle v,w\rangle_d=v^\top w/d$.
We distinguish the update bound $U_{\max}$ from the characteristic
normalized coordinate step $u_\star$; neither is assumed to equal one.

\subsection{Staleness-Induced Bias}
\label{subsec:bias}

The behavior-dependent field in~\eqref{eq:H-def} gives
\begin{equation}
g_t=\mathcal H_{\mathrm{on}}(\theta_t)+\delta_t+\xi_t,
\qquad
\delta_t=\mathcal H(\theta_t,\phi_t)-\mathcal H_{\mathrm{on}}(\theta_t),
\label{eq:effective-update}
\end{equation}
where $\xi_t=g_t-\mathcal H(\theta_t,\phi_t)$ satisfies
$\mathbb E[\xi_t\mid\mathcal F_t]=0$.
In the periodic-synchronization model, $\delta_t=0$ whenever $k_t=0$.

\begin{assumption}[Local regularity]
\label{ass:regularity}\label{ass:bounded_grad}\label{ass:lipschitz}\label{ass:behavior_smooth}
On a working neighborhood $\mathcal U$, the relevant gradients and
updates satisfy $\|\nabla_\theta\ell(z;\theta,\phi)\|\le\Gamma$ and
$\|u_t\|\le U_{\max}$. Moreover,
$D_{\mathrm{KL}}(p_\phi\|p_{\phi'})\le C_\pi\|\phi-\phi'\|^2$ and
$\|\mathbb E_{p_\phi}[\nabla_\theta\ell(z;\theta,\phi)
-\nabla_\theta\ell(z;\theta,\phi')]\|\le L_\rho\|\phi-\phi'\|$.
\end{assumption}

\begin{proposition}[A resettable bias envelope]
\label{thm:scaling}
Under Assumption~\ref{ass:regularity}, let
$C_{\mathrm{stale}}=L_\rho+\Gamma\sqrt{2C_\pi}$.
While the relevant states remain in $\mathcal U$,
\begin{equation}
\|\delta_t\|
\le C_{\mathrm{stale}}\|\theta_t-\phi_t\|
\le C_{\mathrm{stale}}U_{\max}
\sum_{j=t-k_t}^{t-1}\eta_j.
\label{eq:linear-scaling}
\end{equation}
Thus, at constant $\eta$, the bias has a sawtooth \emph{upper envelope}
proportional to $k_t\eta$, with cycle average at most
$C_{\mathrm{stale}}U_{\max}\bar k\eta$.
\end{proposition}

\paragraph{Proof sketch and scope.}
Adding and subtracting the on-policy integrand under $p_{\phi_t}$
splits $\delta_t$ into a surrogate-weighting shift and a sampling shift.
Their norms are bounded by $L_\rho\|\theta_t-\phi_t\|$ and
$\Gamma\sqrt{2C_\pi}\|\theta_t-\phi_t\|$, respectively, the latter by
Pinsker's inequality.
Summing the intervening updates gives~\eqref{eq:linear-scaling}
(Appendix~\ref{app:proof-bias}).
With a locally Lipschitz behavior-policy Jacobian,
$\delta_t=B_t(\theta_t-\phi_t)+O(\|\theta_t-\phi_t\|^2)$,
where $B_t=-\partial_\phi\mathcal H(\theta_t,\theta_t)$.
The bound controls magnitude, not sign: group normalization can
reinforce displacement, while other GRPO terms can oppose it.
Appendix~\ref{app:channels} isolates these contributions; positive
net feedback remains a separate condition.

\subsection{Stability-Boundary Scaling}
\label{subsec:scaling}

Consider a coherent block $\mathcal C$ containing a fraction $f$ of
coordinates, with fixed signs $\chi_j\in\{\pm1\}$ and common raw-noise
scale $\sigma_{\mathcal C}$.
Its RMS-unit direction is $e=\chi/\sqrt f$, with $\chi=0$ outside
$\mathcal C$.
A normalized signed coordinate step $u_\star s$ then gives drift speed
$\sqrt f\,u_\star s$ along $e$.
If the stale response reinforces displacement with coefficient $b>0$,
its mean-lag signal-to-noise ratio is
$b\bar k\eta u_\star s/\sigma_{\mathcal C}$.

\begin{assumption}[Persistent mean-field feedback]
\label{ass:aligned}\label{ass:saturating}
The drift direction persists across synchronization cycles, and its
cycle-level coherence obeys the mean-lag closure
\begin{equation}
s_{n+1}=\Phi(\gamma s_n+\rho_n),
\qquad
\gamma=\frac{b u_\star}{\sigma_{\mathcal C}}\bar k\eta,
\qquad 0\le s_n\le1,
\label{eq:self-consistency}
\end{equation}
where $\rho_n\ge0$ tends to zero.
The response $\Phi$ is odd, increasing, $L_\Phi$-Lipschitz,
$\Phi'(0)>0$, and $0<\Phi(r)<1$ for $r>0$, with $\Phi(r)\to1$.
\end{assumption}

The gain $\gamma$ measures stale signal relative to raw-gradient noise.
For example, replacing a coordinate's second moment by its stationary
value gives $\Phi(r)=r/\sqrt{1+r^2}$: its normalized mean saturates even
as the signal grows.
This motivates the bounded response, not an exact identity for AdamW.
The closure also requires directional memory across resets; neither
condition follows from the bias bound (Appendix~\ref{app:adam}).

\begin{proposition}[Mean-field stability boundary]
\label{prop:threshold}
Define
\begin{equation}
\gamma_c=\inf_{r>0}\frac{r}{\Phi(r)},
\qquad
c_1=\frac{\sigma_{\mathcal C}\gamma_c}{b u_\star}.
\label{eq:threshold-variational}
\end{equation}
For $\gamma<\gamma_c$, the zero-signal map has no positive fixed
point and $s_n\to0$.
For $\gamma>\gamma_c$, its largest fixed point $s_+(\gamma)>0$
is nondecreasing and tends to one as $\gamma\to\infty$.
If $s_{n_0}\ge s_+(\gamma)$, then $s_n\ge s_+(\gamma)$ thereafter
and $s_n\to s_+(\gamma)$.
The mean-field boundary is therefore
\begin{equation}
\bar k\eta=c_1,
\qquad
\eta_{\mathrm{mf}}(S)=\frac{2c_1}{S-1},\quad S>1.
\label{eq:safe-region}
\end{equation}
\end{proposition}

\paragraph{Origin of the threshold.}
At a positive fixed point, $r=\gamma s$ gives
$\gamma=r/\Phi(r)$, excluding such points below $\gamma_c$.
Above $\gamma_c$, some $s\in(0,1)$ satisfies $\Phi(\gamma s)>s$;
since $\Phi(\gamma)<1$, continuity gives a positive root.
For comparable response parameters across $S$, the resulting boundary
yields law~(i), $\eta_{\mathrm{mf}}(S)\propto(S-1)^{-1}\approx S^{-1}$,
conditional on entry into the persistent branch.
Convergence, onset types, and finite-$S$ corrections are treated in
Appendix~\ref{app:proof-threshold}.

\subsection{Collapse-Time Scaling}
\label{subsec:collapse}

\paragraph{Snapshot exit model.}
To connect drift to the reward-defined event, introduce a local
surrogate-validity region
$\mathcal N_R=\{\theta:\|\theta-\theta_0\|<R\}$, $R>0$, and its first
\emph{synchronized-snapshot exit},
\begin{equation}
\tau_S\triangleq\inf\{mS:\theta_{mS}\notin\mathcal N_R\}.
\label{eq:snapshot-exit}
\end{equation}
The modeling correspondence is $t_{\mathrm{obs}}=\tau_S$:
the first failed behavior snapshot is the first snapshot outside
$\mathcal N_R$.
This is distinct from the empirical collapse criterion and avoids
identifying delayed reward logs with learner-update times.

\paragraph{From coherence to parameter motion.}
Write $a_t=\mathbb E[u_t\mid\mathcal F_t]$ and
$D_t=\sum_{j<t}\eta_j a_j$ for cumulative conditional
\emph{optimizer} drift.
At constant $\eta$, the cycle-average conditional mean
$\bar a_n=S^{-1}\sum_{k=0}^{S-1}a_{nS+k}$ connects the scalar
coherence to parameter motion through
\begin{equation}
\bar a_n=\sqrt f\,u_\star s_n e+r_n,
\qquad
v=\sqrt f\,u_\star s_{\mathrm{sat}}.
\label{eq:main-coherence-bridge}
\end{equation}
Here $r_n$ collects motion outside the coherent block, directional
variation, and optimizer effects omitted by the closure.
If $s_n$ approaches $s_{\mathrm{sat}}$ and residual motion is small,
summing cycles yields the cumulative-drift condition below; the initial
transient and incomplete cycles enter $a$.
Appendix~\ref{app:proof-clock} quantifies this bridge from coherence to
approximately constant \emph{mean velocity}.

For deterministic nonnegative learning rates and a horizon
$T\in S\mathbb N$, let
$V_T=\sum_{j<T}\eta_j^2$ and
$h_S=\max_{0\le m<T/S}\sum_{j=mS}^{(m+1)S-1}\eta_j$.
For $\delta\in(0,1)$, a uniform martingale bound gives the noise radius
$q_T(\delta)=2U_{\max}\sqrt{2V_T\log(2/\delta)}$
(Appendix~\ref{app:proof-clock}).

\begin{theorem}[Snapshot-clock bounds]
\label{thm:collapse}
Suppose $\|u_t\|\le U_{\max}$ before $T\wedge\tau_S$.
For an RMS-unit direction $e$ and constants $v>0$, $a\ge0$,
and $0\le\varepsilon<1$, assume
\begin{equation}
\|D_t-v\Lambda_t e\|\le a+\varepsilon v\Lambda_t,
\qquad t\le T\wedge\tau_S.
\label{eq:coherent-tube}
\end{equation}
Let $E_T=a+q_T(\delta)<R$ and assume
$(1-\varepsilon)v\Lambda_T\ge R+E_T$.
Then, with probability at least $1-\delta$, $\tau_S\le T$ and
\begin{equation}
\frac{R-E_T}{(1+\varepsilon)v}
\le\Lambda_{\tau_S}
\le\frac{R+E_T}{(1-\varepsilon)v}+h_S.
\label{eq:t-collapse-bound}
\end{equation}
\end{theorem}

\paragraph{Proof sketch.}
Decompose $\theta_t-\theta_0=D_t+N_t$ into conditional drift and
centered-update fluctuations.
A uniform martingale bound controls $\|N_t\|\le q_T(\delta)$ for
all $t\le T\wedge\tau_S$, giving
\begin{equation}
(1-\varepsilon)v\Lambda_t-E_T
\le\|\theta_t-\theta_0\|
\le(1+\varepsilon)v\Lambda_t+E_T.
\label{eq:main-radial-sandwich}
\end{equation}
The upper envelope prevents an earlier exit; the lower envelope forces
exit by its first scheduled crossing of $R$, with at most $h_S$
additional cumulative learning rate.
Appendix~\ref{app:proof-clock} proves both bounds on the same uniform
event, including at the random exit time.

\paragraph{Saturation removes lag from the leading clock.}
At constant learning rate, $h_S=S\eta$, so
\begin{equation}
\eta t_{\mathrm{obs}}
=c_2\!\left[1+O\!\left(\varepsilon+\frac{a+q_T(\delta)}{R}\right)\right]
+O(S\eta),
\qquad
c_2=\frac{R}{\sqrt f\,u_\star s_{\mathrm{sat}}}.
\label{eq:clock}
\end{equation}
When the effective exit distance and saturated speed are comparable
across $S$, this yields law~(ii), with explicit transient, noise, and
synchronization-resolution errors.
Before saturation, the leading budget instead scales as
$R/(\sqrt f\,u_\star s_+(\gamma))$ and can depend on $S$.
Thus a shared clock reflects a nearly constant distance-to-speed ratio,
not the existence of a coherent branch alone.
Exact snapshot quantization and its resolution error are given in
Appendix~\ref{app:proof-clock}.

\paragraph{Why stable runs can outlive the clock.}
If cumulative mean drift instead remains bounded by $a_0<R$, the same
concentration argument prevents snapshot exit whenever
$a_0+q_T(\delta)<R$.
At constant $\eta$ and fixed confidence and margin, this permits an
$\Omega(\eta^{-2})$ non-exit horizon rather than the coherent
$\eta^{-1}$ clock.
Appendix~\ref{app:proof-clock} gives the explicit bound and sufficient
conditions controlling residual drift; decay of scalar coherence alone
is not enough.

\subsection{Interpreting the Gradient Cosine}
\label{subsec:readout}

Let $\mu_t=\mathcal H(\theta_t,\phi_t)$,
$\sigma_{t,j}^2=\mathbb E[\xi_{t,j}^2\mid\mathcal F_t]>0$,
$\rho_{t,j}=\mu_{t,j}/\sigma_{t,j}$, and
$w_{t,j}=\sigma_{t,j}^2/\sum_i\sigma_{t,i}^2$.
When adjacent means vary slowly and the relevant norms and inner
products concentrate, Appendix~\ref{app:cos} gives
\begin{equation}
c_t\triangleq\cos(g_t,g_{t-1})
\approx\frac{\mathrm{SNR}_t}{1+\mathrm{SNR}_t},
\qquad
\mathrm{SNR}_t=\sum_j w_{t,j}\rho_{t,j}^2.
\label{eq:cos-snr}
\end{equation}
Unlike coordinate-normalized updates, this readout weights each
coordinate by its share of raw-gradient noise.

A block can therefore sustain coherent normalized updates while
contributing little to the raw cosine
(Proposition~\ref{prop:readout}, Appendix~\ref{app:cos}).
A subsequent rise additionally requires a stronger raw-gradient signal
near failure.
A thin high-SNR boundary layer can yield a short warning interval,
conditional on the readout and trajectory assumptions; its bound is
given in Appendix~\ref{app:cos}.
Thus a low raw-gradient cosine alone does not certify stability.

\subsection{Practical Calibration}
\label{subsec:practical}

The two laws concern different aspects of a run:
$\bar k\eta$ controls the feedback regime, while $R/v$ controls the
failure clock on a persistent branch.
A coarse sweep can bracket the empirical stability boundary and
estimate a cumulative-rate budget from several failed snapshots.
Increasing $S$ therefore calls for an approximately inverse change
in $\eta$, followed by validation at the intended horizon.
The budget is not a universal limit on stable training, nor does a
single failed run determine it exactly.
For varying learning rates, the clock interpretation applies only
while the drift conditions persist; changing $\eta$ can also change
the feedback regime.
For $S=1$, synchronization-induced bias vanishes, but other sources
of optimization failure remain possible.

\section{Experiments}
\label{sec:experiments}

We examine the collapse laws through joint $(S,\eta)$ sweeps on two
Llama models, complemented by extended-horizon stability tests and
a cross-family study of collapse timing.

\subsection{Experimental Setup}
\label{subsec:setup}

\paragraph{Models and data.}
We train Llama-3.2-1B-Instruct and
Llama-3.2-3B-Instruct~\citep{grattafiori2024llama3} on the
7,500-problem training split of MATH~\citep{hendrycks2021math}
and evaluate on its 5,000-problem test split.
For Qwen3-8B~\citep{yang2025qwen3}, we train on
DAPO-Math-17k~\citep{yu2025dapo} and report
$\mathrm{avg}@1$ on MATH-500~\citep{lightman2024verify}
and Minerva Math~\citep{lewkowycz2022minerva}, and
$\mathrm{avg}@8$ on AMC 2023/2024~\citep{aops_amc}.

\paragraph{Protocol.}
We use the asynchronous pipeline of Section~\ref{subsec:prelim-async},
implemented in verl~\citep{sheng2024hybridflow} with
vLLM~\citep{kwon2023vllm} rollout workers synchronized every $S$ learner
updates. All runs use vanilla GRPO without additional staleness correction
(Appendix~\ref{app:implementation}).
The main experiment crosses $S\in\{8,16,32\}$ with constant
$\eta\in\{10^{-6},5\times10^{-7},2\times10^{-7},10^{-7}\}$ for both
Llama models. We record the collapse index $M_{\mathrm{collapse}}$
and policy-indexed time $t_{\mathrm{obs}}=SM_{\mathrm{collapse}}$,
and compute the gradient cosine $c_t$ following \citet{xu2026gac}.
Main-sweep runs stop after collapse is identified or after a collapse-free
horizon exceeding the empirical reference clock $c_2/\eta$.
The main sweeps use one seed per configuration; targeted 3B repeats at
$S=32$ are reported in Appendix~\ref{app:multiseed}.

\begin{table}[t]
\centering
\caption{Collapse indices $M_{\mathrm{collapse}}$ for the Llama sweeps and
Qwen3-8B extension. ``Stable'' denotes no observed collapse within the training horizon;
$t_{\mathrm{obs}}$ rows give the corresponding learner-update times
for collapsing configurations.}
\label{tab:collapse}

\small
\setlength{\tabcolsep}{4pt}
\begin{tabular}{@{}llcccc@{}}
\toprule
Model & $S$ & $\eta=10^{-6}$ & $5\times10^{-7}$ & $2\times10^{-7}$ & $10^{-7}$\\
\midrule
\multirow{4}{*}{Llama-3.2-1B / 3B}
 & $8$  & $4$ & $8$ & stable & stable\\
 & $16$ & $2$ & $4$ & $10$   & stable\\
 & $32$ & $1$ & $2$ & $5$    & $10$\\
\cmidrule(lr){2-6}
 & $t_{\mathrm{obs}}$ & $32$ & $64$ & $160$ & $320$\\
\midrule
\multirow{2}{*}{Qwen3-8B}
 & $16$ & $4$ & $8$ & $20$ & stable\\
\cmidrule(lr){2-6}
 & $t_{\mathrm{obs}}$ & $64$ & $128$ & $320$ & ---\\
\bottomrule
\end{tabular}
\end{table}

\begin{figure}[t]
    \centering
    \includegraphics[width=\linewidth]{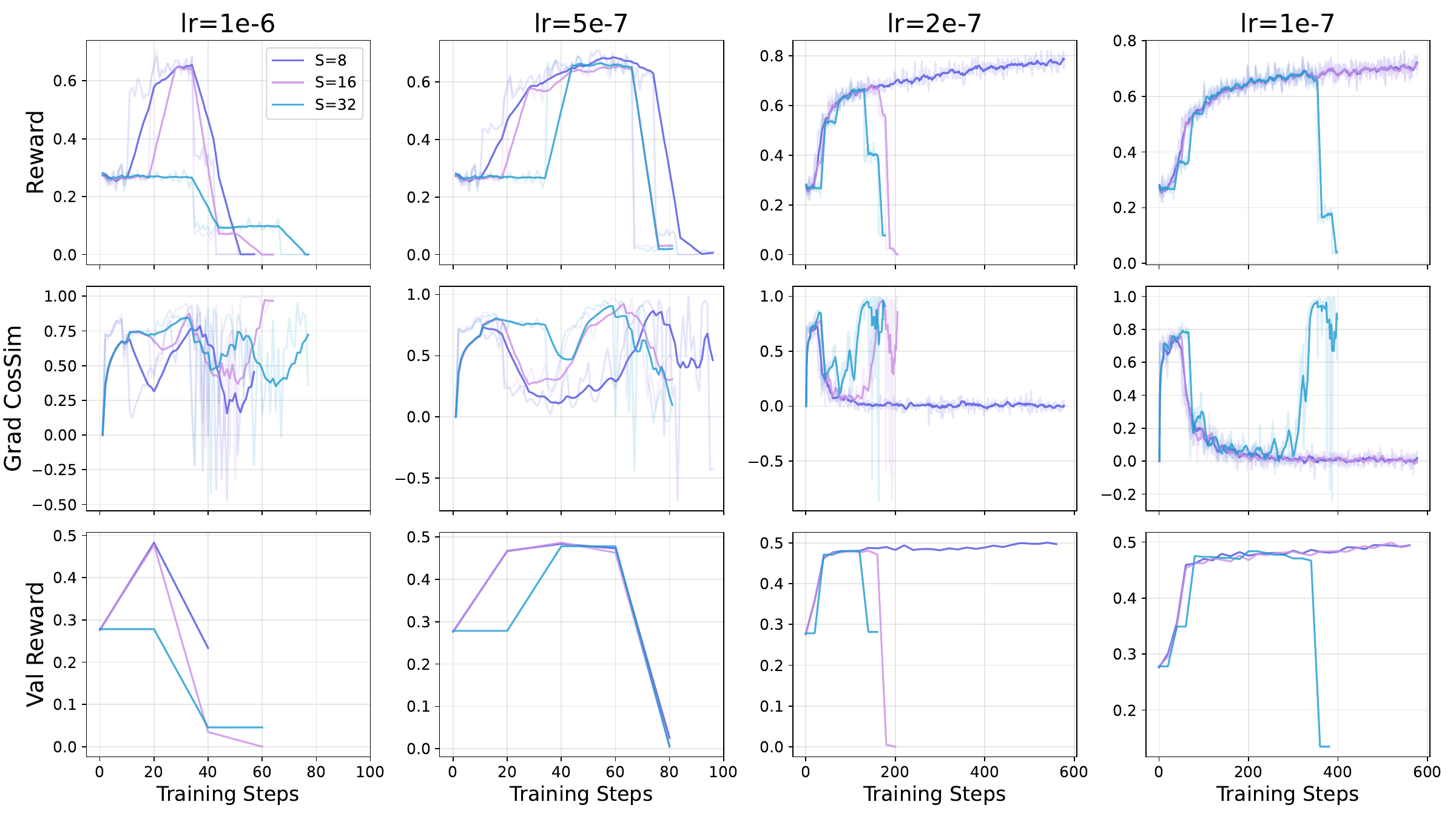}
    \caption{Llama-3.2-3B-Instruct $(S,\eta)$ sweep.
    Columns vary $\eta$; curves vary $S$.
    Rows show training reward, gradient cosine $c_t$, and MATH
    validation reward.
    On the tested grid, runs with $S\eta\leq1.6\times10^{-6}$ remain
    stable; others have reported collapse times
    $t_{\mathrm{obs}}=3.2\times10^{-5}/\eta$, independent of $S$
    (Table~\ref{tab:collapse}).
    At $\eta\geq5\times10^{-7}$, $c_t$ remains elevated through much
    of training; near the boundary, collapsing runs can initially
    resemble stable ones before $c_t$ rises near collapse.
    The 1B sweep has identical collapse indices
    (Figure~\ref{fig:metrics_grid_1b}).
    Horizontal-axis ranges are column-specific.}
    \label{fig:metrics_grid_3b}
\end{figure}

\subsection{A Stability Boundary in \texorpdfstring{$S\eta$}{S eta}}
\label{subsec:phase}

Both Llama sizes exhibit identical stability patterns
(Table~\ref{tab:collapse}), with joint dynamics in
Figures~\ref{fig:metrics_grid_3b} and~\ref{fig:metrics_grid_1b}
and per-configuration reward traces in Appendix~\ref{app:extra-figures}.
Reducing $\eta$ expands the stable region: all three
synchronization intervals collapse at $\eta\geq5\times10^{-7}$;
only $S=8$ is stable at $2\times10^{-7}$; and both $S=8$ and $16$
are stable at $10^{-7}$, while $S=32$ still collapses.
In the 3B sweep, unstable runs can initially improve training
and validation reward before losing these gains, whereas stable runs
sustain them.
Across the tested grid, stable configurations satisfy
$S\eta\leq1.6\times10^{-6}$ and collapsing configurations have
$S\eta\geq3.2\times10^{-6}$.
Thus, doubling $S$ from $8$ to $16$ halves the largest tested stable
learning rate, supporting the approximately constant-$S\eta$ boundary
predicted by Proposition~\ref{prop:threshold}.

\begin{figure}[!htbp]
    \centering
    \includegraphics[width=\linewidth]{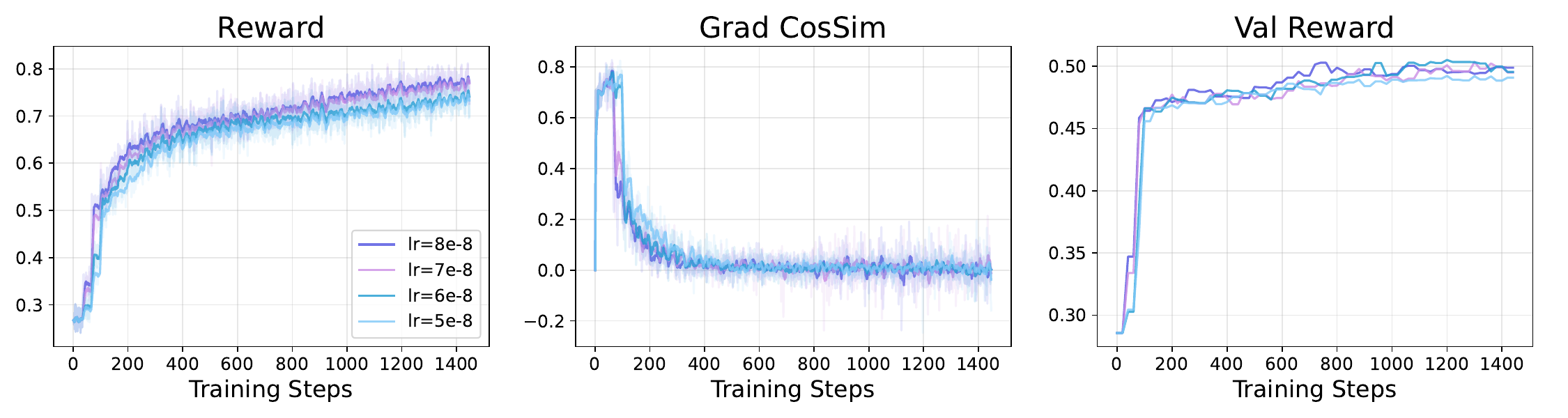}
    \caption{Extended-horizon Llama-3.2-3B-Instruct runs at $S=32$ with
    $\eta\in\{8,7,6,5\}\times10^{-8}$.
    Left to right: training reward, gradient cosine $c_t$, and
    MATH validation reward.
    All four runs remain collapse-free for approximately $1{,}400$
    learner updates, exceeding twice their reference times $c_2/\eta$,
    with $c_t$ settling near zero.}
    \label{fig:smalllr_3b}
\end{figure}

\subsection{A Cumulative-Learning-Rate Clock for Collapse}
\label{subsec:escape}

All nine collapsing configurations for each Llama size satisfy
$t_{\mathrm{obs}}\eta=c_2=3.2\times10^{-5}$.
The reported collapse times are $32$, $64$, $160$, and $320$ updates
for $\eta=10^{-6},5\times10^{-7},2\times10^{-7},10^{-7}$, respectively.
Within each learning-rate column, these times are unchanged across the
values of $S$ that collapse. For example, at $\eta=10^{-6}$,
$S=8,16,32$ yield collapse indices $4,2,1$, respectively, but all
correspond to $t_{\mathrm{obs}}=32$.
Larger staleness therefore reduces the number of synchronizations before
collapse without shortening the recorded learner-update horizon.
Even configurations at $S\eta=3.2\times10^{-6}$ share this clock with
settings with up to ten times larger $S\eta$.
For the 3B model at $S=32$, all three tested seeds reproduce the same
four recorded collapse indices (Table~\ref{tab:seed}).
At the policy-snapshot resolution, these observations agree with
Proposition~\ref{prop:threshold} and Theorem~\ref{thm:collapse}
(\eqref{eq:clock}): $S\eta$ separates the regimes, while the cumulative
learning rate governs collapse timing within the unstable regime.

\subsection{Gradient Alignment as a Late Warning}
\label{subsec:readout-exp}

The cosine trajectories show distinct temporal patterns
(Figure~\ref{fig:metrics_grid_3b}, middle row).
Stable runs can exhibit an initial alignment transient during the reward
increase, after which $c_t$ approaches zero as reward remains high.
At $\eta\geq5\times10^{-7}$, collapsing runs instead show elevated,
though fluctuating, cosine over much of their short trajectories.
Near the boundary, at $S=32$ and $\eta=10^{-7}$, $c_t$ first decays
toward the stable-run level, then rises near the reward drop across all
three tested seeds (Figure~\ref{fig:seed_highlr}).
Thus, early alignment need not signal failure, and a prolonged period
of low cosine does not establish stability. The late rise is consistent
with Proposition~\ref{prop:readout} and Corollary~\ref{cor:precursor} (Appendix~\ref{app:cos}),
but offers limited warning near the boundary.

\subsection{Extended-Horizon Stability}
\label{subsec:coherence}

To distinguish stability from delayed collapse, we train the 3B Llama model
at $S=32$ with $\eta\in\{8,7,6,5\}\times10^{-8}$ for approximately
$1{,}400$ learner updates.
All four runs remain collapse-free (Figure~\ref{fig:smalllr_3b}):
training reward continues to improve after its initial rise, validation
reward settles near $0.5$, and $c_t$ approaches zero without a sustained
late resurgence.
Each run exceeds twice its reference time $c_2/\eta$, whereas
$\eta=10^{-7}$ at the same $S$ collapses after $320$ updates.
This survival is not explained by extrapolating the unstable-regime clock
to smaller learning rates.
Separate three-seed checks remain collapse-free through $500$ updates
for all four rates (Figure~\ref{fig:seed_smalllr}), exceeding the
reference clock at $8\times10^{-8}$ and $7\times10^{-8}$.
Stability at $\eta=8\times10^{-8}$ and collapse at $\eta=10^{-7}$
further bracket the transition between $S\eta=2.56\times10^{-6}$ and
$3.2\times10^{-6}$.
Survival beyond the reference clock is consistent with the conditional
non-collapse bound of Corollary~\ref{cor:incoherent}.

\begin{figure}[!htbp]
    \centering
    \includegraphics[width=\linewidth]{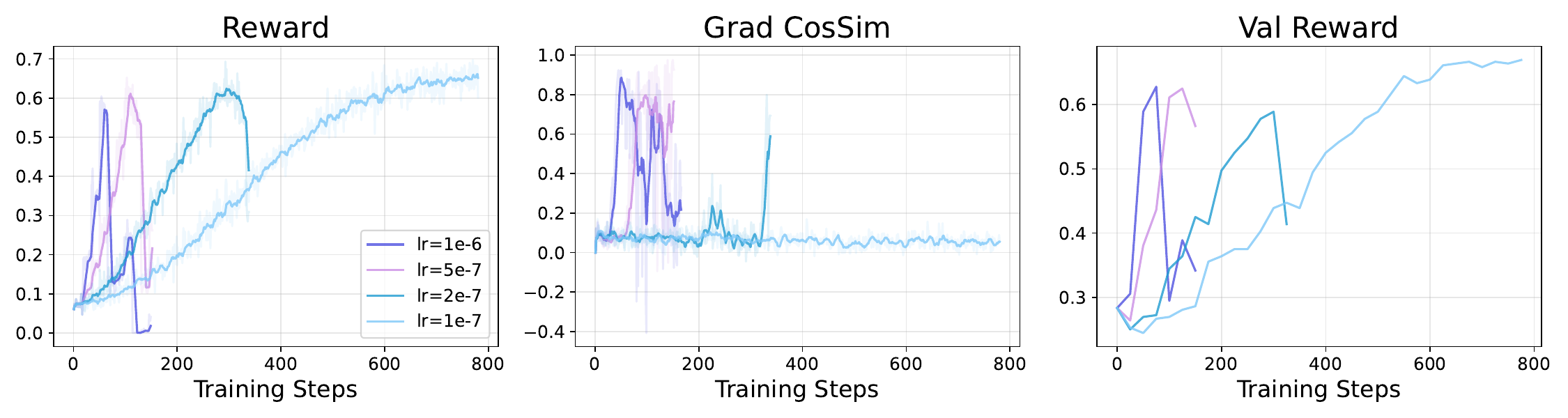}
    \caption{Qwen3-8B at $S=16$, with one curve per learning rate.
    Left to right: training reward, raw-gradient cosine $c_t$, and mean
    validation accuracy.
    The $\eta=10^{-7}$ run remains collapse-free over the displayed
    horizon, whereas the three larger learning rates have reported
    collapse times satisfying $t_{\mathrm{obs}}\eta=6.4\times10^{-5}$
    (Table~\ref{tab:collapse}).
    Near the boundary ($\eta=2\times10^{-7}$), $c_t$ stays near the
    stable-run level before rising sharply around collapse.}
    \label{fig:metric_8b}
\end{figure}

\subsection{A Fixed-Staleness Test on Qwen3-8B}
\label{subsec:qwen8b}

We test Qwen3-8B at $S=16$ using the main learning-rate grid.
The three larger rates show rapid early reward gains followed by abrupt
drops and declining validation accuracy (Figure~\ref{fig:metric_8b}).
Their policy-indexed collapse times, $t_{\mathrm{obs}}=64,128,320$,
satisfy $t_{\mathrm{obs}}\eta=6.4\times10^{-5}$, twice the Llama
calibration (Table~\ref{tab:collapse}).
By contrast, $\eta=10^{-7}$ learns more gradually but remains
collapse-free for approximately $800$ updates, beyond its reference
time of $640$, with training reward and validation accuracy reaching
about $0.65$ and $0.67$, respectively.
At $\eta=2\times10^{-7}$, the cosine stays near the stable-run level
until rising near collapse, reproducing the late-warning pattern.
This experiment supports inverse-learning-rate timing in a different
model--data setting with a different clock constant; independence from
$S$ is tested by the Llama sweeps.

\section{Related Work}
\label{sec:related-work}

\paragraph{Staleness in asynchronous RL for LLMs.}
Recent systems improve throughput by decoupling rollout generation from
learning \citep{noukhovitch2025asynchronous,fu2025areal,piche2025pipelinerl,
song2026echo2}. They manage policy lag through empirical bounds
\citep{fu2025areal,khatri2025art} or objective-level corrections
\citep{li2025a3po,zheng2025prosperity,tian2026mugrpo,xu2026gac}.
We retain vanilla GRPO and characterize how admissible staleness depends
on the learning rate, providing a baseline for evaluating such corrections.

\paragraph{Guarantees for delayed and off-policy updates.}
Distributed RL mitigates policy mismatch through importance weighting
\citep{munos2016safe,espeholt2018impala}.
Delayed-SGD and off-policy GRPO analyses relate convergence or improvement
to delay, step size, and policy proximity
\citep{lian2015asynchronous,mishchenko2022asynchronous,
mroueh2026revisiting,yao2026grouprelative}.
We complement these guarantees with an approximately constant-$S\eta$
stability boundary and a $1/\eta$ collapse clock, insensitive to $S$
across the tested unstable configurations.

\paragraph{Optimizer dynamics.}
Adam's coordinatewise normalization can reduce sensitivity to gradient
magnitude \citep{balles2018dissecting,kunstner2023noise}, while transformer
gradient scales vary by parameter block \citep{zhang2024transformers}.
Our feedback model combines these properties to explain the
cumulative-learning-rate clock and late gradient-cosine rise near collapse
(Appendix~\ref{app:related}).
\section{Conclusion}
\label{sec:conclusion}

We identify two empirical scaling laws that distinguish the stability of
asynchronous GRPO from its collapse timing. In the Llama sweeps, the
stability boundary is approximately characterized by constant $S\eta$,
while collapsing runs follow an inverse-learning-rate clock with an
approximately shared cumulative-rate budget across synchronization
intervals. Qwen3-8B supports the timing relationship at fixed $S$ with a
different budget. Extended-horizon runs remain collapse-free beyond the
fitted clock, while gradient cosine provides only a late warning near the
boundary. Our local $O(S\eta)$ bias bound and conditional dynamical model
connect these observations to staleness feedback, persistent drift, and
optimizer normalization. The practical implication is to calibrate $S$
and $\eta$ jointly for each model and training setup, then validate the
selected configuration over the intended horizon rather than mistaking
delayed collapse for stability.

\section*{Reproducibility Statement}

Section~\ref{subsec:prelim-async} defines the synchronization protocol,
collapse criterion, and timing convention; Section~\ref{subsec:setup}
specifies the models, data, and experimental grid.
Appendix~\ref{app:implementation} lists all training and evaluation
settings, stopping rules, and the gradient-cosine measurement.
Multi-seed checks and additional curves are in
Appendices~\ref{app:multiseed} and~\ref{app:extra-figures}, and all
assumptions and proofs are in Appendix~\ref{app:proofs}.
All models and datasets are publicly available.

\section*{Ethics Statement}

This work studies the optimization stability of reinforcement learning
for language models on public mathematical-reasoning benchmarks.
It involves no human subjects, personal data, or newly collected data,
and it introduces no new model capabilities.
By helping practitioners avoid collapsing configurations, the findings
may reduce wasted computation.

\section*{AI Use Statement}

Large language models assisted with literature search, drafting and
editing, checking derivations, and \LaTeX{} formatting. The authors
designed and ran all experiments, verified all AI-assisted content,
and take full responsibility for the paper.

\bibliography{iclr2027_conference}
\bibliographystyle{iclr2027_conference}
\clearpage
\appendix
\section{Theory Details and Proofs}
\label{app:proofs}

This appendix separates three logical steps.
The local bias bound is a consequence of regularity of the GRPO field.
The feedback threshold is a property of an explicitly specified
mean-field model.
The collapse clock follows from concentration around cumulative
\emph{optimizer} drift, together with a reward-to-snapshot-exit
correspondence.
None of the latter assumptions is inferred solely from agreement
between the tabulated collapse times and a fitted clock.
Throughout, $\|v\|=\|v\|_2/\sqrt d$,
$\langle v,w\rangle_d=v^\top w/d$, and matrix operator norms are induced
by this norm. Ordinary dot products and $\|\cdot\|_2$ are used when
stated explicitly.

\subsection{Geometry, Conditional Means, and the Observation Model}
\label{app:discussion}

\paragraph{The quantities being separated.}
The raw-gradient mean is
$\mu_t=\mathcal H(\theta_t,\phi_t)$, while the conditional mean of the
applied optimizer update is $a_t=\mathbb E[u_t\mid\mathcal F_t]$.
These are different for a nonlinear, stateful optimizer.
The first determines the bias $\delta_t$ and raw-gradient readout;
the second determines parameter displacement.
The pre-sampling filtration $\mathcal F_t$ includes the learner,
behavior policy, and optimizer state.
The sampling model in the preliminaries assumes a fresh mini-batch
conditional on this state; using cached or adaptively selected data
would require a corresponding change to the conditional field.

\paragraph{RMS geometry.}
For $\mathcal C\subseteq\{1,\ldots,d\}$ with $|\mathcal C|=fd$, let
$\chi_j\in\{\pm1\}$ on $\mathcal C$ and $\chi_j=0$ elsewhere.
Then $e=\chi/\sqrt f$ satisfies $\|e\|=1$ and
$|\langle u,e\rangle_d|\le\|u\|$.
If the signed conditional coordinate means are $u_\star s$ on
$\mathcal C$ and zero elsewhere, then
\begin{equation}
a= u_\star s\chi
=\sqrt f\,u_\star s\,e,
\qquad
\|a\|=\langle a,e\rangle_d=\sqrt f\,u_\star s.
\label{eq:appendix-block-speed}
\end{equation}
The constants $U_{\max}$ and $u_\star$ serve different purposes:
the former bounds stochastic steps, while the latter is a typical
coordinate scale in the drift model.
The absence of an explicit factor $d$ in the resulting inequalities
does not make their constants independent of model dimension.

\paragraph{Working domain versus failure region.}
The regularity neighborhood $\mathcal U$ need not equal
$\mathcal N_R$.
The former is a domain on which the field can be controlled; the
latter is a model of the effective reward-failure boundary.
Snapshot exit may occur after a learner has crossed that boundary
between synchronizations.
At constant $\eta$, the ball of radius $R+S\eta U_{\max}$ suffices
to cover an intervening cycle starting inside $\mathcal N_R$, provided
the update bound holds there.
All local statements require the relevant learner and behavior
parameters to belong to $\mathcal U$.
Smoothness of a clipped, finite-sample surrogate is not automatic:
Assumption~\ref{ass:regularity} is imposed on the averaged field in
the region used by the analysis.

\paragraph{Policy-indexed observations.}
The empirical event is the abrupt training-reward degradation defined
in Section~\ref{subsec:prelim-async}.
The exit model identifies its generating snapshot with the first
snapshot outside $\mathcal N_R$:
\begin{equation}
t_{\mathrm{obs}}=\tau_S,
\qquad
M_{\mathrm{collapse}}=\tau_S/S.
\label{eq:observation-correspondence}
\end{equation}
This is a modeling correspondence, not a reward-labeling rule or a
claim that failure is irreversible.
It refers to the synchronization step of the generating policy,
not to when its reward is delivered to the logger.

For comparison, define the learner's first exit
$\tau=\inf\{t:\theta_t\notin\mathcal N_R\}$.
In general, only $\tau\le\tau_S$ follows: a learner may exit and
re-enter before a synchronization.
If the first exit remains visible at the next synchronization, then
\begin{equation}
\tau_S=S\lceil\tau/S\rceil,
\qquad
\tau\in(\tau_S-S,\tau_S].
\label{eq:rollout-to-step}
\end{equation}
The main snapshot-clock theorem does not require this additional
no-re-entry condition.
In particular, an empirical timestamp must not be rounded according
to a fitted theoretical clock to enforce~\eqref{eq:rollout-to-step}.

\subsection{Proof of the Local Bias Bound}
\label{app:proof-bias}

Fix $\theta,\phi,\phi'\in\mathcal U$ and set
$h(z)=\nabla_\theta\ell(z;\theta,\phi')$.
Adding and subtracting $\mathbb E_{p_\phi}h$ gives
\begin{align}
\mathcal H(\theta,\phi)-\mathcal H(\theta,\phi')
&=\mathbb E_{p_\phi}
 [\nabla_\theta\ell(z;\theta,\phi)
  -\nabla_\theta\ell(z;\theta,\phi')]
\nonumber\\
&\quad +(\mathbb E_{p_\phi}-\mathbb E_{p_{\phi'}})h(z).
\label{eq:appendix-bias-split}
\end{align}
The first term is bounded by $L_\rho\|\phi-\phi'\|$.
For the second, boundedness of $h$ and Pinsker's inequality yield
\begin{align*}
\|(\mathbb E_{p_\phi}-\mathbb E_{p_{\phi'}})h\|
&\le2\Gamma\,\operatorname{TV}(p_\phi,p_{\phi'})\\
&\le\Gamma\sqrt{2D_{\mathrm{KL}}(p_\phi\|p_{\phi'})}
\le\Gamma\sqrt{2C_\pi}\|\phi-\phi'\|.
\end{align*}
Hence $\phi\mapsto\mathcal H(\theta,\phi)$ is locally Lipschitz
with constant $C_{\mathrm{stale}}$.
Taking $\phi=\phi_t$ and $\phi'=\theta=\theta_t$, and using
$\theta_t-\phi_t=\sum_{j=t-k_t}^{t-1}\eta_j u_j$, proves
Proposition~\ref{thm:scaling}.
For constant $\eta$, averaging the bound over a complete cycle uses
$S^{-1}\sum_{k=0}^{S-1}k=(S-1)/2$.
Only the bound is sawtooth-shaped; the realized bias need not grow
monotonically within a cycle. \hfill$\square$

\paragraph{Directional expansion.}
Suppose additionally that the behavior-policy Jacobian of
$\mathcal H$ exists and is locally $L_B$-Lipschitz.
With $\Delta=\theta-\phi$, integration along the segment from
$\theta$ to $\phi$ gives
\begin{equation}
\mathcal H(\theta,\phi)-\mathcal H(\theta,\theta)
=B(\theta)\Delta+r(\theta,\phi),
\quad
B(\theta)=-\partial_\phi\mathcal H(\theta,\theta),
\quad
\|r(\theta,\phi)\|\le\tfrac12 L_B\|\Delta\|^2.
\label{eq:appendix-bias-taylor}
\end{equation}
The Lipschitz field bound gives $\|B(\theta)\|_{\mathrm{op}}
\le C_{\mathrm{stale}}$.
Differentiability alone would give an $o(\|\Delta\|)$ remainder,
not the displayed quadratic bound.
The sign of $\langle B(\theta)\Delta,\Delta\rangle_d$ remains
unrestricted.

\subsection{GRPO Channels: Exact Cancellation, Clipping, and Finite Groups}
\label{app:channels}

These examples isolate individual channels rather than represent the
full token-level, clipped, KL-regularized objective.
They show why reinforcement of drift is possible, and why its net
sign must remain a separate premise.

\paragraph{Exact importance weighting.}
For a behavior-independent advantage $A(x,y)$, a sequence-level ratio
$\rho(y)=\pi_\theta(y\mid x)/\pi_\phi(y\mid x)$, and no clipping,
\begin{equation}
\mathbb E_{y\sim\pi_\phi}
 [\rho(y)A(x,y)\nabla_\theta\log\pi_\theta(y\mid x)]
=
\mathbb E_{y\sim\pi_\theta}
 [A(x,y)\nabla_\theta\log\pi_\theta(y\mid x)].
\label{eq:appendix-exact-is}
\end{equation}
Thus this isolated policy-gradient term has no stale-policy bias.
This conclusion does not automatically extend to a sampled KL term,
behavior-dependent group statistics, or token-level ratios that leave
the prefix distribution uncorrected.

\paragraph{The clipping contribution.}
At a fixed stale sampling distribution, compare the clipped and
unclipped policy-gradient terms.
Away from clipping boundaries, their difference is
\begin{equation}
\delta^{\mathrm{clip}}
=-\mathbb E_{p_\phi}\sum_{\mathrm{clipped}}w_{i,l}
  \hat A_i\rho_{i,l}\nabla_\theta
  \log\pi_\theta(y_{i,l}\mid x,y_{i,<l}),
\label{eq:appendix-clip-exact}
\end{equation}
where $w_{i,l}>0$ is the response/token normalization.
Clipping occurs when $\hat A_i>0$, $\rho_{i,l}>1+\epsilon$, or
$\hat A_i<0$, $\rho_{i,l}<1-\epsilon$.
In either case $\hat A_i\log\rho_{i,l}>0$.
If the score changes smoothly along $\Delta=\theta-\phi$, then
\[
\Delta^\top\nabla_\theta\log\pi_\theta
=\log\rho_{i,l}+O(\|\Delta\|^2),
\]
with the norm convention absorbed into the local remainder constant.
Assuming the weighted absolute coefficients have bounded expectation,
\begin{equation}
\Delta^\top\delta^{\mathrm{clip}}
\le O(\|\Delta\|^2).
\label{eq:clip-dissipative}
\end{equation}
Clipping therefore removes updates that, to this local accuracy,
push further along the already favored displacement.
This is not a global negative-semidefiniteness claim.
For a fixed nonzero clip width and uniformly bounded scores,
clipping is inactive in a sufficiently small neighborhood of
$\theta=\phi$; its derivative there is zero, rather than necessarily
a nonzero component of $B(\theta)$.

\paragraph{An exact finite-group normalization calculation.}
Fix a prompt, let rewards be in $\{0,1\}$, and write
$q_\theta=\Pr_{\pi_\theta}(r=1)\in(0,1)$.
Consider $G\ge2$ independent responses, sequence-level importance
ratios, no clipping, no length weighting, and no KL term.
Use the within-group population standard deviation and stabilizer
$\epsilon_A\ge0$; all-identical groups have zero advantage.
For a selected response let $K$ be the number of successful peers,
so $K\sim\operatorname{Bin}(G-1,q_\phi)$.
Its normalized advantage conditional on success or failure is
\begin{equation}
A_+(k)=\frac{G-1-k}{\sqrt{(k+1)(G-1-k)}+G\epsilon_A},
\qquad
A_-(k)=\frac{-k}{\sqrt{k(G-k)}+G\epsilon_A},
\label{eq:appendix-finite-advantages}
\end{equation}
with zero-over-zero endpoints interpreted as zero.
Define the polynomial
\begin{equation}
C_{G,\epsilon_A}(p)
=\sum_{k=0}^{G-1}\binom{G-1}{k}p^k(1-p)^{G-1-k}
  [A_+(k)-A_-(k)].
\label{eq:appendix-finite-coefficient}
\end{equation}

\begin{lemma}[Finite-group stale normalization]
\label{lem:finite-group}
For the isolated estimator above,
\begin{equation}
\mathcal H_G(\theta,\phi)
=C_{G,\epsilon_A}(q_\phi)\nabla_\theta q_\theta.
\label{eq:appendix-finite-field}
\end{equation}
Consequently, if $q$ is twice continuously differentiable,
\begin{equation}
\delta_G
=-C'_{G,\epsilon_A}(q_\theta)
  (\nabla_\theta q_\theta)(\nabla_\theta q_\theta)^\top\Delta
+O(\|\Delta\|^2).
\label{eq:adv-channel}
\end{equation}
\end{lemma}

\begin{proof}
Exchangeability reduces the group mean to a selected response.
Importance weighting changes that response's law to $\pi_\theta$
while leaving its peers under $\pi_\phi$.
The score identities
\[
\mathbb E_{\pi_\theta}[\mathbf1\{r=1\}\nabla\log\pi_\theta]
=\nabla q_\theta,
\qquad
\mathbb E_{\pi_\theta}[\mathbf1\{r=0\}\nabla\log\pi_\theta]
=-\nabla q_\theta
\]
then give $[A_+(K)-A_-(K)]\nabla q_\theta$ conditional on the peers.
Averaging over $K$ proves~\eqref{eq:appendix-finite-field}.
Subtract the on-policy expression and expand
$q_\phi=q_\theta-\nabla q_\theta^\top\Delta+O(\|\Delta\|^2)$
to obtain~\eqref{eq:adv-channel}.
\end{proof}

The channel is reinforcing where $C'_{G,\epsilon_A}(q)<0$ and
opposing where that derivative is positive.
Finite groups cannot be replaced by their informative-group
probability alone: for example,
$C_{2,\epsilon_A}(p)=1/(1+2\epsilon_A)$ is constant, so this
channel vanishes identically at $G=2$ in the stated estimator.
For $G=3$, it is also constant,
$C_{3,\epsilon_A}(p)=2/(\sqrt2+3\epsilon_A)$.
These examples preclude a universal finite-group reinforcement claim.

\paragraph{Population normalization.}
In the population-statistics variant, with
$\varsigma(q)=\sqrt{q(1-q)}$, the coefficient is
$C_\infty(q)=1/(\varsigma(q)+\epsilon_A)$.
The first-order response becomes
\begin{equation}
B^{\mathrm{pop}}
=\frac{1-2q}{2\varsigma(q)(\varsigma(q)+\epsilon_A)^2}
  (\nabla q)(\nabla q)^\top.
\label{eq:appendix-population-channel}
\end{equation}
It is positive semidefinite for $q<1/2$ and negative semidefinite
for $q>1/2$.
This provides an explicit reinforcing channel on difficult prompts,
not a sign theorem for the full GRPO field.
Prefix-distribution shifts, clipping, finite-group effects, and the
sampled KL penalty can change the net response.

\subsection{Proof and Scope of the Feedback Threshold}
\label{app:proof-threshold}

Let $F_\gamma(s)=\Phi(\gamma s)$ on $[0,1]$.
Continuity, monotonicity, $F_\gamma(0)=0$, and $F_\gamma(1)<1$
ensure that a largest fixed point $s_+(\gamma)$ exists.

\paragraph{Variational threshold.}
At any positive fixed point, $r=\gamma s$ satisfies
$\gamma=r/\Phi(r)$.
Therefore no such point exists when
$\gamma<\gamma_c=\inf_{r>0}r/\Phi(r)$.
Conversely, for $\gamma>\gamma_c$, choose $r>0$ with
$r/\Phi(r)<\gamma$.
Then $s=r/\gamma<\Phi(r)<1$ and $F_\gamma(s)>s$.
Because $F_\gamma(1)<1$, continuity gives a positive fixed point
between $s$ and one.
The Lipschitz bound and differentiability at zero give
\begin{equation}
\frac{1}{L_\Phi}\le\gamma_c
\le\gamma_{\mathrm{lin}}\triangleq\frac{1}{\Phi'(0)}.
\label{eq:appendix-threshold-bracket}
\end{equation}
No statement at equality $\gamma=\gamma_c$ is needed.

\paragraph{Monotonicity and saturation.}
If $\gamma'>\gamma$, then
$F_{\gamma'}(s_+(\gamma))\ge s_+(\gamma)$.
A root of $F_{\gamma'}(s)-s$ therefore lies at or above
$s_+(\gamma)$, proving monotonicity of the largest root.
For any fixed $s_0\in(0,1)$,
$F_\gamma(s_0)\to1>s_0$ as $\gamma\to\infty$.
It follows that $s_+(\gamma)\ge s_0$ for sufficiently large $\gamma$,
and hence $s_+(\gamma)\to1$.

\paragraph{Dynamics with a vanishing task signal.}
If $s_n\ge s_+$, then
$s_{n+1}=\Phi(\gamma s_n+\rho_n)\ge F_\gamma(s_+)=s_+$.
This proves forward invariance above the largest root.
Let $L=\limsup_n s_n$ and choose a subsequence along which
$s_{n_k+1}\to L$ and $s_{n_k}\to L'\le L$.
Since $\rho_n\to0$, continuity gives $L=F_\gamma(L')\le F_\gamma(L)$.
Above the largest root, $F_\gamma(s)<s$; therefore $L\le s_+$.
Together with invariance, this yields $s_n\to s_+$.
Below threshold, $F_\gamma(s)<s$ for every $s>0$, so the same
limsup argument gives $s_n\to0$ for every initialization.
This proves Proposition~\ref{prop:threshold}. \hfill$\square$

\paragraph{Branch entry is a separate condition.}
The result does not assert that every high-lag run reaches the upper
branch. A persistent task signal or optimizer memory can seed a
common direction, but their presence alone does not prove that
$s_{n_0}\ge s_+(\gamma)$.
Without cross-cycle directional persistence, coherent individual
cycles can cancel in their cumulative displacement.
The closure in \eqref{eq:self-consistency} is therefore a mean-field
closure with an explicit basin condition, not an exact reduction
of the AdamW recursion.

\paragraph{Continuous versus discontinuous onset.}
If $\Phi$ is strictly concave on $[0,\infty)$, then
$F_\gamma(s)-s$ is strictly concave and
$\gamma_c=\gamma_{\mathrm{lin}}$.
For $\gamma>\gamma_c$ it has one positive fixed point; that point
approaches zero as $\gamma\downarrow\gamma_c$.
To see continuity, any positive subsequential limit would be a
positive fixed point at $\gamma_c$, where strict concavity gives
$F_{\gamma_c}(s)<s$ for $s>0$.

More generally, if $\gamma_c<\gamma_{\mathrm{lin}}$, then the onset
is discontinuous.
Indeed, differentiability at zero gives an $s_0>0$ and a neighborhood
of $\gamma_c$ in which $F_\gamma(s)<s$ for $0<s\le s_0$.
Every positive fixed point just above threshold must exceed $s_0$.
This proves the possibility of a jump, not its occurrence in the
LLM experiments.
The observed budget constrains the combined ratio $R/v$, and cannot
identify coherence alone without comparable exit geometry,
transients, and measurement precision.

\paragraph{An explicit nonlinear example.}
For $r\ge0$, consider the effective response
\begin{equation}
\Psi_\kappa(r)=
\frac{r+\kappa r^3}{\sqrt{1+(r+\kappa r^3)^2}},\qquad\kappa>0,
\label{eq:appendix-nonlinear-response}
\end{equation}
extended oddly.
It is increasing, bounded, globally Lipschitz, and
$\Psi_\kappa'(0)=1$.
Choose $\gamma_0<1$ and $s_0\in(0,1)$.
If
\[
\kappa>
\frac{s_0/\sqrt{1-s_0^2}-\gamma_0s_0}{(\gamma_0s_0)^3},
\]
then $\Psi_\kappa(\gamma_0s_0)>s_0$, so
$\gamma_c<\gamma_0<\gamma_{\mathrm{lin}}=1$ and a discontinuous
branch exists.
This is an illustrative composite of superlinear feedback and a
saturating optimizer response; it is not derived from the measured
GRPO response or used to fit the reported collapse labels.

\paragraph{Retaining every lag.}
Keeping the full lag distribution in the homogeneous closure gives,
for $S>1$,
\begin{equation}
F_{S,\gamma}(s)
=\frac1S\sum_{k=0}^{S-1}
 \Phi\!\left(\frac{2\gamma k}{S-1}s\right).
\label{eq:sawtooth-map}
\end{equation}
This is a lag-resolved closure, not an exact stochastic optimizer
recursion.
Its derivative at zero is still $\gamma\Phi'(0)$.
Its largest-root threshold $\gamma_c(S)$ lies in
$[1/L_\Phi,1/\Phi'(0)]$ by the same contraction and
intermediate-value arguments.
For a strictly concave response, $\gamma_c(S)=\gamma_{\mathrm{lin}}$;
for a nonconcave response the nonlinear threshold can depend on $S$.
The saturated root tends to $(S-1)/S$, because the $k=0$ contribution
is zero. Persistent optimizer memory can supply a nonzero reset-step
contribution, but that is an additional model, not an automatic
property of~\eqref{eq:sawtooth-map}.

At fixed $\gamma$, let
$F_{\infty,\gamma}(s)=\int_0^1\Phi(2\gamma us)\,du$.
A left-Riemann-sum bound, using the $2L_\Phi\gamma$ Lipschitz
constant in $u$, gives
\begin{equation}
\sup_{s\in[0,1]}
|F_{S,\gamma}(s)-F_{\infty,\gamma}(s)|
\le\frac{L_\Phi\gamma+1}{S}.
\label{eq:appendix-lag-discretization}
\end{equation}
Thus a common constant-product boundary is a useful leading-order
reduction, not an exact finite-$S$ identity for every response.
Replacing $\bar k$ by $S/2$ adds the further relative correction
$1/S$.

\subsection{Uniform Drift Bounds and the Snapshot Clock}
\label{app:proof-clock}

We first establish one uniform event and then evaluate it at the
random exit time. This avoids substituting a stopping time into
a bound proved only at a fixed time.
Learning rates are deterministic and nonnegative, and
$\delta\in(0,1)$ is a fixed failure probability. The constant-rate
experiments are a special case.

\begin{lemma}[A uniform stochastic tube]
\label{lem:uniform-tube}
Let $\zeta$ be a stopping time and suppose
$\|u_t\|\le U_{\max}$ for $t<T\wedge\zeta$.
With $a_t=\mathbb E[u_t\mid\mathcal F_t]$ and
$D_t=\sum_{j<t}\eta_j a_j$, with probability at least $1-\delta$,
\begin{equation}
\sup_{0\le t\le T\wedge\zeta}
\|\theta_t-\theta_0-D_t\|
\le q_T(\delta)
=2U_{\max}\sqrt{2V_T\log(2/\delta)}.
\label{eq:appendix-uniform-tube}
\end{equation}
\end{lemma}

\begin{proof}
Use the stopped martingale
$N_t=\sum_{j<t}\eta_j\mathbf1\{j<\zeta\}(u_j-a_j)$.
The indicator is $\mathcal F_j$-measurable, and each increment has
RMS norm at most $2U_{\max}\eta_j$.
The Hilbert-space maximal inequality of
\citet[Theorem~3.5]{pinelis1994optimum}, applied to the scaled inner
product $\langle\cdot,\cdot\rangle_d$, gives
\[
\Pr\!\left(\max_{t\le T}\|N_t\|\ge r\right)
\le2\exp\!\left(-\frac{r^2}{8U_{\max}^2V_T}\right).
\]
For $t\le\zeta$, this martingale equals
$\theta_t-\theta_0-D_t$.
Choosing $r=q_T(\delta)$ proves the claim.
When $V_T=0$, all displacements vanish and the result is immediate.
\end{proof}

\paragraph{A sharper fluctuation radius.}
If the centered updates satisfy
$\|u_t-a_t\|\le\sigma_u$ before stopping, the same argument replaces
$q_T$ by $\sigma_u\sqrt{2V_T\log(2/\delta)}$ throughout the clock
bounds. The default choice $\sigma_u=2U_{\max}$ needs no extra
information, but can be conservative because $U_{\max}$ also bounds
the mean motion. A smaller noise bound must be established separately,
not selected to match the observed collapse times.

\paragraph{Proof of Theorem~\ref{thm:collapse}.}
Apply Lemma~\ref{lem:uniform-tube} with $\zeta=\tau_S$.
On its event, for every $t\le T\wedge\tau_S$,
\begin{equation}
(1-\varepsilon)v\Lambda_t-E_T
\le\|\theta_t-\theta_0\|
\le(1+\varepsilon)v\Lambda_t+E_T,
\qquad E_T=a+q_T(\delta).
\label{eq:appendix-radial-sandwich}
\end{equation}
Let $\bar t$ be the first synchronization step with
$(1-\varepsilon)v\Lambda_{\bar t}\ge R+E_T$.
The horizon condition ensures $\bar t\le T$.
If $\tau_S>\bar t$, the lower bound at $\bar t$ gives
$\|\theta_{\bar t}-\theta_0\|\ge R$, contradicting the definition
of $\tau_S$.
Thus $\tau_S\le\bar t\le T$.
At exit, the upper radial bound gives
$R\le(1+\varepsilon)v\Lambda_{\tau_S}+E_T$, proving the lower
clock bound.
By minimality of $\bar t$, its preceding synchronization has
cumulative rate strictly below $(R+E_T)/((1-\varepsilon)v)$.
The intervening cycle contributes at most $h_S$, so
\[
\Lambda_{\tau_S}\le\Lambda_{\bar t}
\le\frac{R+E_T}{(1-\varepsilon)v}+h_S.
\]
Both bounds hold on the same event. \hfill$\square$

\paragraph{Normalized error and finite resolution.}
Writing $c_2=R/v$, the theorem is equivalently
\begin{equation}
\frac{1-E_T/R}{1+\varepsilon}
\le\frac{\Lambda_{\tau_S}}{c_2}
\le\frac{1+E_T/R}{1-\varepsilon}+\frac{h_S}{c_2}.
\label{eq:appendix-relative-clock}
\end{equation}
For $\varepsilon$ bounded away from one and small $E_T/R$, this
gives~\eqref{eq:clock} under the observation correspondence.
At constant $\eta$ and a comparison horizon $T=O(c_2/\eta)$,
\[
\frac{q_T(\delta)}{R}
=O\!\left(\frac{U_{\max}}{v}
           \sqrt{\frac{\eta\log(2/\delta)}{c_2}}\right).
\]
The error is controlled by normalized noise, not by parameter count
alone. The bound can be loose, especially when collapse occurs after
only one or a few synchronization cycles.
In particular, $S\eta/c_2$ need not be small in such cells.

For an exactly straight deterministic trajectory
$\theta_t-\theta_0=v t\eta e$, the snapshot model reduces to
\begin{equation}
M_{\mathrm{collapse}}
=\left\lceil\frac{c_2}{S\eta}\right\rceil,
\qquad
c_2\le\eta t_{\mathrm{obs}}<c_2+S\eta.
\label{eq:m-collapse-bound}
\end{equation}
This exact ceiling formula is a special case, not an identity
obtained by discarding the error terms in~\eqref{eq:clock}.

\paragraph{Connecting cycle coherence to cumulative drift.}
At constant $\eta$, let
$\bar a_n=S^{-1}\sum_{k=0}^{S-1}a_{nS+k}$.
The bridge from the mean-field model to actual parameter motion is
\begin{equation}
\bar a_n=\sqrt f\,u_\star s_n e+r_n,
\label{eq:appendix-coherence-bridge}
\end{equation}
where $r_n$ includes other coordinates, directional variation,
weight decay, and deviations of AdamW from the scalar closure.
This is an explicit approximation condition on conditional
\emph{optimizer} means, not on raw gradients.
Suppose that from a cycle $n_0$ onward
$|s_n-s_{\mathrm{sat}}|\le\varepsilon_s s_{\mathrm{sat}}$ and
$\|r_n\|\le\varepsilon_r v$, where
$v=\sqrt f\,u_\star s_{\mathrm{sat}}$.
Then
$\|\bar a_n-v e\|\le(\varepsilon_s+\varepsilon_r)v$.
Writing $t_0=n_0S$, summing complete cycles and bounding a partial
cycle yields~\eqref{eq:coherent-tube} with
\begin{equation}
\varepsilon=\varepsilon_s+\varepsilon_r,
\qquad
a=(U_{\max}+v)(\Lambda_{t_0}+S\eta).
\label{eq:appendix-transient-residual}
\end{equation}
An incomplete cycle at a time before $\tau_S$ uses only
pre-exit updates; $\tau_S$ itself is a cycle boundary.
The initial segment is absorbed into $a$ without assuming it
already follows the persistent direction.
If the mean-field branch converges but does not saturate, the
appropriate velocity is instead
$v(S,\eta)=\sqrt f\,u_\star s_+(\gamma)$, which generally produces
an $S$-dependent budget.
A shared clock therefore needs comparable velocity and exit distance,
not just the existence of a nonzero branch.

\begin{corollary}[A sufficient non-collapse horizon]
\label{cor:incoherent}
If $\sup_{t\le T\wedge\tau_S}\|D_t\|\le a_0<R$ and
$a_0+q_T(\delta)<R$, then
$\Pr(\tau_S>T)\ge1-\delta$.
In particular, at constant $\eta$ this holds whenever
\begin{equation}
T<\frac{(R-a_0)^2}
        {8U_{\max}^2\eta^2\log(2/\delta)}.
\label{eq:incoherent-bound}
\end{equation}
\end{corollary}

\begin{proof}
On the uniform-tube event,
$\|\theta_t-\theta_0\|\le a_0+q_T(\delta)<R$
for every $t\le T\wedge\tau_S$.
If $\tau_S\le T$, this contradicts the defining exit inequality.
Substitute $V_T=T\eta^2$ for the constant-rate bound.
\end{proof}

For fixed confidence and a margin $R-a_0$ bounded away from zero,
this is a high-probability $\Omega(\eta^{-2})$ lower bound on the
snapshot exit time.
The zero-conditional-mean case has $a_0=0$.
A condition $\|a_t\|\le\varepsilon U_{\max}$ instead gives
$\|D_t\|\le\varepsilon U_{\max}\Lambda_t$; at
$T=O(\eta^{-2})$ this term has order $\varepsilon/\eta$.
Thus merely assuming that $\varepsilon$ eventually tends to zero
is not enough.

\paragraph{When the scalar subcritical model controls cumulative drift.}
A sufficient, stronger condition is
$q=\gamma L_\Phi<1$ and $\sum_n\rho_n<\infty$.
Then $s_{n+1}\le q s_n+L_\Phi\rho_n$, and summing gives
\begin{equation}
\sum_{n\ge0}s_n
\le\frac{s_0+L_\Phi\sum_{n\ge0}\rho_n}{1-q}.
\label{eq:appendix-summable-coherence}
\end{equation}
Together with~\eqref{eq:appendix-coherence-bridge} and bounded
cumulative residuals, this bounds the coherent contribution to
$D_t$ by $S\eta\sqrt f\,u_\star\sum_n s_n$, up to a partial cycle.
It closes a sufficient route from a subcritical feedback model to
Corollary~\ref{cor:incoherent}; the entire subcritical region is not
automatically a zero-drift regime.

\paragraph{Displacement predictions.}
The coherent tube predicts an approximately linear displacement in
$\Lambda_t$ when $a$, $\varepsilon$, and the martingale term are small.
With negligible cumulative mean drift, the bound is instead of
order $\sqrt{V_t}=\sqrt{\eta\Lambda_t}$ at constant rate.
This is an upper scale, not a universal equality for every stable run.
For an unstopped zero-mean process, orthogonality of martingale
differences gives
$\mathbb E\|\theta_t-\theta_0\|^2
=\sum_{j<t}\eta_j^2\mathbb E\|u_j\|^2$;
a matching mean-square diffusion scale additionally requires
nondegenerate update variance.
Reward curves alone do not measure either displacement regime.

\subsection{AdamW: Bounds, Normalization, and Momentum}
\label{app:adam}

AdamW uses the moment recursions of Adam with decoupled weight
decay~\citep{kingma2015adam,loshchilov2019decoupled}.
To keep time indexing explicit, write its ascent update after $n$
gradient observations as
\begin{equation}
\nu_n=
\frac{\widehat m_n}{\sqrt{\widehat v_n}+\epsilon_{\mathrm{Adam}}}
-\lambda\theta_{n-1},
\qquad
\widehat m_n=\sum_{r=1}^n\alpha_{n,r}h_r,
\qquad
\widehat v_n=\sum_{r=1}^n\omega_{n,r}h_r^2,
\label{eq:appendix-adam}
\end{equation}
where $0\le\beta_1<1$, $0<\beta_2<1$,
$\epsilon_{\mathrm{Adam}}>0$, $\lambda\ge0$, and $h_r$ is the
gradient after any gradient-norm clipping,
all operations are coordinatewise, and
\[
\alpha_{n,r}=\frac{(1-\beta_1)\beta_1^{n-r}}{1-\beta_1^n},
\qquad
\omega_{n,r}=\frac{(1-\beta_2)\beta_2^{n-r}}{1-\beta_2^n}.
\]
This paragraph uses $\nu_n$ for the applied update to distinguish
one-based optimizer indexing from $u_t$ in the main text.

\paragraph{A corrected worst-case bound.}
Cauchy--Schwarz gives, for every coordinate and every gradient history,
\begin{equation}
\frac{|\widehat m_{n,j}|}{\sqrt{\widehat v_{n,j}}
      +\epsilon_{\mathrm{Adam}}}
\le C_n,
\qquad
C_n^2=
\frac{(1-\beta_1)^2}{1-\beta_2}
\frac{1-\beta_2^n}{(1-\beta_1^n)^2}
\sum_{k=0}^{n-1}\left(\frac{\beta_1^2}{\beta_2}\right)^k.
\label{eq:appendix-adam-bound}
\end{equation}
When $\widehat v_{n,j}=0$, the numerator is also zero and the
inequality holds by continuity.
If $0<\beta_2<1$ and $\beta_1^2<\beta_2$, then
$C_\star=\sup_{n\ge1}C_n<\infty$.
On a parameter region $\|\theta\|\le B_\theta$, it follows that
\begin{equation}
\|\nu_n\|\le C_\star+\lambda B_\theta.
\label{eq:appendix-adam-rms-bound}
\end{equation}
Thus a local $U_{\max}$ can include both bias correction and weight
decay. It depends on the parameter region as well as optimizer
hyperparameters, and is not the typical saturated velocity.
To prove~\eqref{eq:appendix-adam-bound}, apply
$(\sum_r\alpha_r h_r)^2
\le(\sum_r\alpha_r^2/\omega_r)(\sum_r\omega_r h_r^2)$ and
substitute the two weights.
Finiteness follows from the geometric sum and convergence of the
bias-correction factors.

\paragraph{What saturation means.}
For a constant nonzero coordinate gradient $h$, bias correction gives
$\widehat m=h$, $\widehat v=h^2$; its normalized component is
$h/(|h|+\epsilon_{\mathrm{Adam}})$, approaching a sign step only when
$|h|\gg\epsilon_{\mathrm{Adam}}$.
For Gaussian $g=\mu+\sigma Z$, a sign update has the exact mean
$\mathbb E[\operatorname{sign}(g)]
=\operatorname{erf}(\mu/(\sqrt2\sigma))$.
Replacing Adam's random second moment by a fixed stationary value
$\mu^2+\sigma^2$ gives the illustrative response
\begin{equation}
\frac{\mathbb E[g]}{\sqrt{\mathbb E[g^2]}}
=\frac{\rho}{\sqrt{1+\rho^2}},\qquad\rho=\mu/\sigma.
\label{eq:appendix-stationary-response}
\end{equation}
This is exact for that deterministic-denominator approximation, not
for AdamW's conditional mean with a random denominator and stored
momentum. The distinction between sign, magnitude, and variance
adaptation motivates this modeling choice~\citep{balles2018dissecting}.

\paragraph{Clipping and scale invariance.}
The history-wise bound holds whether or not $h_r$ is clipped.
With $\epsilon_{\mathrm{Adam}}=0$, multiplying \emph{every} gradient
in a coordinate's history by the same positive constant leaves its
normalized Adam component unchanged.
Gradient-norm clipping generally applies a different factor at each
step, so it can change the trajectory and its mean velocity.
Weight decay can likewise alter both direction and speed.
Neither operation justifies replacing $U_{\max}$ by one or treating
$v$ as invariant without qualification.

\paragraph{What a momentum bound actually establishes.}
For a fixed normalization, a constant learning rate, and
$z_t=\beta_1 z_{t-1}+(1-\beta_1)h_t$ with $z_{-1}=0$,
\begin{equation}
\eta\sum_{t=0}^{T-1}z_t
=\eta\sum_{t=0}^{T-1}h_t
-\eta\sum_{t=0}^{T-1}\beta_1^{T-t}h_t.
\label{eq:appendix-momentum-sum}
\end{equation}
If $\|h_t\|\le U_h$, the final term has norm at most
$\eta U_h\beta_1/(1-\beta_1)$.
Bias-correcting this first-moment average adds at most another term
of that size: since $\|z_t\|\le(1-\beta_1^{t+1})U_h$,
$\|z_t/(1-\beta_1^{t+1})-z_t\|\le\beta_1^{t+1}U_h$.
These are cumulative-displacement identities along the same input
sequence, not a proof that two independently evolved training
trajectories differ by a fixed number of steps.
A time-varying Adam preconditioner or learning rate needs additional
control and is not covered by this identity.

The main stochastic tube uses $a_t=\mathbb E[u_t\mid\mathcal F_t]$
for the actual optimizer, so it remains valid with momentum under its
stated hypotheses.
However, zero-mean raw gradients do not imply $a_t=0$:
the optimizer state contains past noise.
This is why the incoherent result explicitly requires control of
cumulative optimizer drift.

\subsection{Raw-Gradient Cosine: An Error-Controlled Readout}
\label{app:cos}

Let $\Sigma_t=\mathbb E[\xi_t\xi_t^\top\mid\mathcal F_t]$ and
$Q_t=\|\mu_t\|_2^2+\operatorname{tr}\Sigma_t>0$.
The following deterministic perturbation lemma states exactly what
is required to replace the observed cosine by a signal fraction.

\begin{lemma}[Cosine approximation]
\label{lem:cosine-error}
For $j=t-1,t$, suppose
$|\|g_j\|_2^2-Q_j|\le\epsilon_N Q_j$, with $0\le\epsilon_N<1$,
and
$|g_t^\top g_{t-1}-\mu_t^\top\mu_{t-1}|
\le\epsilon_I\sqrt{Q_tQ_{t-1}}$.
Then
\begin{equation}
\left|c_t-
\frac{\mu_t^\top\mu_{t-1}}{\sqrt{Q_tQ_{t-1}}}\right|
\le\frac{\epsilon_I+\epsilon_N}{1-\epsilon_N}.
\label{eq:appendix-cosine-error}
\end{equation}
If the displayed signal-over-energy ratio differs from
$\mathrm{SNR}_t/(1+\mathrm{SNR}_t)$ by at most $\epsilon_D$,
add $\epsilon_D$ to the right-hand side.
\end{lemma}

\begin{proof}
Write $r=\mu_t^\top\mu_{t-1}/\sqrt{Q_tQ_{t-1}}$, so $|r|\le1$.
The normalized numerator is $r+e_I$, $|e_I|\le\epsilon_I$;
the normalized denominator is
$d_N=\|g_t\|_2\|g_{t-1}\|_2/\sqrt{Q_tQ_{t-1}}$
lies in $[1-\epsilon_N,1+\epsilon_N]$.
Hence
$|(r+e_I)/d_N-r|
\le(\epsilon_I+|r|\epsilon_N)/(1-\epsilon_N)$.
The final claim is the triangle inequality.
\end{proof}

\paragraph{Why the hypotheses matter.}
The exact numerator expansion is
\[
g_t^\top g_{t-1}
=\mu_t^\top\mu_{t-1}
 +\mu_t^\top\xi_{t-1}+\xi_t^\top g_{t-1}.
\]
The last term has zero conditional mean given $\mathcal F_t$.
The middle term need not: $\mu_t$ depends on the preceding update
and can also change at a synchronization.
Thus fresh sampling alone does not establish the inner-product
hypothesis or slow variation across a reset.
Likewise, a ratio of expected energies is not generally the
expectation of a random cosine.
Lemma~\ref{lem:cosine-error} isolates these errors rather than
silently discarding them.

When adjacent means and energies are close, $\epsilon_D$ is small,
and the signal ratio becomes
\[
\frac{\|\mu_t\|_2^2}{\|\mu_t\|_2^2+\operatorname{tr}\Sigma_t}
=\frac{\mathrm{SNR}_t}{1+\mathrm{SNR}_t}.
\]
For positive coordinate noise variances,
$\mathrm{SNR}_t=\sum_jw_{t,j}\rho_{t,j}^2$ is an exact identity;
zero-variance coordinates can instead be handled directly through
$\|\mu_t\|_2^2/\operatorname{tr}\Sigma_t$.
\begin{proposition}[Noise-weighted masking]
\label{prop:readout}
If $|\rho_{t,j}|\le\rho_{\mathcal C}$ on $\mathcal C$ and
$|\rho_{t,j}|\le\rho_{\mathcal H}$ elsewhere, then
\begin{equation}
\mathrm{SNR}_t\le
w_{\mathcal C}\rho_{\mathcal C}^2+
(1-w_{\mathcal C})\rho_{\mathcal H}^2,
\label{eq:readout-bound}
\end{equation}
where $w_{\mathcal C}=\sum_{j\in\mathcal C}w_{t,j}$.
Thus a small weighted sum can mask coherent normalized updates
on $\mathcal C$.
\end{proposition}

\begin{proof}
Splitting $\mathrm{SNR}_t=\sum_jw_{t,j}\rho_{t,j}^2$ over
$\mathcal C$ and its complement gives
\[
\sum_{j\in\mathcal C}w_{t,j}\rho_{t,j}^2
+\sum_{j\notin\mathcal C}w_{t,j}\rho_{t,j}^2
\le w_{\mathcal C}\rho_{\mathcal C}^2
+(1-w_{\mathcal C})\rho_{\mathcal H}^2.
\]
Under the cosine-approximation conditions, a small right-hand side
therefore gives a small raw-gradient cosine without excluding large
coordinate-normalized signals in $\mathcal C$.
\end{proof}

\paragraph{Masking is possible, not automatic.}
For example, the algebraic values
$w_{\mathcal C}=10^{-6}$, $\rho_{\mathcal C}=10$, and
$\rho_{\mathcal H}=0$ give $\mathrm{SNR}=10^{-4}$, even though
$\rho_{\mathcal C}/\sqrt{1+\rho_{\mathcal C}^2}$ is close to one.
This is an illustration, not a measurement from the experiments.
Low raw variance alone does not guarantee high normalized stale
signal: the response matrix $B_t$ can scale with the same coordinates.
Localization of coherent drift in a low-noise block requires evidence
or an explicit model assumption.

\paragraph{A conditional late-warning model.}
To model the observed rise, suppose the cosine approximation has
error at most $\epsilon_c$ and there is a thin layer
$R-\Delta_R\le\|\theta_t-\theta_0\|<R$ in which
$\mathrm{SNR}_t\ge\kappa^2$.
This is an additional signal assumption; surrogate failure does not
by itself imply a large raw-gradient signal or an aligned adjacent mean.

\begin{corollary}[Boundary-layer warning]
\label{cor:precursor}\label{ass:edge}
At times in this layer where the readout conditions hold,
$c_t\ge\kappa^2/(1+\kappa^2)-\epsilon_c$.
Suppose, at constant $\eta$, that
$\|\theta_t-\theta_0-vt\eta e\|\le E<R-\Delta_R$ through
$\tau_S$, with a horizon long enough for snapshot exit.
Let $t_{\mathrm{edge}}$ be the first layer-entry-or-crossing time.
Then
\begin{equation}
0\le\tau_S-t_{\mathrm{edge}}
\le\frac{\Delta_R+2E}{v\eta}+S.
\label{eq:appendix-warning-time}
\end{equation}
\end{corollary}

\begin{proof}
Monotonicity of $z/(1+z)$ gives the cosine bound.
At first entry or crossing,
$vt_{\mathrm{edge}}\eta\ge R-\Delta_R-E$.
The first synchronization with $vt\eta\ge R+E$ is necessarily
outside the region, giving $v\tau_S\eta\le R+E+vS\eta$.
Subtracting proves~\eqref{eq:appendix-warning-time}.
Since reaching radius $R$ requires first reaching radius
$R-\Delta_R$, the time difference is nonnegative.
\end{proof}

In the small-error limit, the leading warning duration is
$(\Delta_R/R)c_2/\eta+S$.
A sustained warning additionally requires the trajectory not to
leave the high-SNR layer again before failure.
The result neither guarantees a visible precursor at every step
nor applies the readout approximation across arbitrary synchronization
jumps.
The stopping protocol does not provide evidence about recovery or
post-collapse gradient dynamics, which are outside this analysis.

\subsection{Weaker Moment Assumptions}
\label{app:hp_bounds}

If only $\mathbb E[\|u_t\|^2\mid\mathcal F_t]\le U_2^2$ is
available, Jensen's inequality still gives
\[
\mathbb E\|\theta_t-\phi_t\|
\le U_2\sum_{j=t-k_t}^{t-1}\eta_j,
\qquad
\mathbb E\|\theta_t-\theta_0\|\le U_2\Lambda_t.
\]
These are expectation bounds, not deterministic speed limits.
For the stopped martingale in Lemma~\ref{lem:uniform-tube},
Hilbert-space orthogonality gives
$\mathbb E\|N_T\|^2\le U_2^2V_T$.
Doob's maximal inequality then yields
\begin{equation}
\Pr\!\left(\max_{t\le T}\|N_t\|\ge r\right)
\le\frac{U_2^2V_T}{r^2}.
\label{eq:appendix-second-moment-tube}
\end{equation}
Accordingly, the snapshot-clock proof remains valid with
$q_T=U_2\sqrt{V_T/\delta}$, subject to its other assumptions.
Finite second moments alone do not provide the logarithmic-confidence
radius obtained from bounded increments.

\subsection{Calibration, Identifiability, and Experimental Checks}
\label{app:theory-calibration}

\paragraph{What a shared clock identifies.}
In the deterministic snapshot model, a failed cell identifies the
interval
\begin{equation}
\eta(t_{\mathrm{obs}}-S)<c_2\le\eta t_{\mathrm{obs}}.
\label{eq:appendix-clock-interval}
\end{equation}
Several cells constrain $c_2$ through their interval intersection;
with stochastic and model residuals, those intervals must be widened
using~\eqref{eq:t-collapse-bound}.
For the reported Llama counts, the deterministic intervals intersect
in $(2.88,3.20]\times10^{-5}$; for the three Qwen3 cells, in
$(6.08,6.40]\times10^{-5}$.
These are model-based resolution intervals, not confidence intervals
or independently estimated parameter ranges.
Equality of the reported products is therefore compatible with a
common budget but does not determine a unique latent exit time.
It cannot, on its own, establish saturation or a discontinuous
feedback transition.

\paragraph{Scope of the empirical comparison.}
The Llama sweeps vary both $S$ and $\eta$ and therefore examine both
scalings. Qwen3-8B varies $\eta$ at $S=16$, so it checks timing at
that synchronization interval, not invariance across $S$.
The datasets and evaluation setups also differ; different budgets
cannot be attributed solely to model family.
A finite-horizon stable label is not a proof of infinite-horizon
stability or of the martingale idealization.
The main-grid products bracket an empirical boundary; the extended
3B comparison at $S=32$ uses products $2.56\times10^{-6}$ and
$3.20\times10^{-6}$, or mean-lag products
$1.24\times10^{-6}$ and $1.55\times10^{-6}$.
A threshold inferred from these points is specific to the measured
setup and must account for the unequal observation horizons.

\paragraph{Falsifiable links in the mechanism.}
The required checks are distinct.
A stale-versus-fresh field comparison at matched learner states tests
the sign of $\langle\delta_t,\theta_t-\phi_t\rangle_d$.
Applied-update projections and cycle-to-cycle directions test
persistence and saturated mean speed; raw-gradient cosine is not a
substitute for either.
RMS displacement against $\Lambda_t$ tests the straight-drift tube
and whether stable runs require a different drift model.
Reward annotations and snapshot-version records test the observation
correspondence independently of the fitted clock.
These checks are predictions and validation requirements of the
mechanism, not experiments asserted to have been completed.
\providecommand{\tbd}[1]{\textcolor{red}{[#1]}}

\section{Implementation Details}
\label{app:implementation}

\paragraph{System.}
All runs use the asynchronous recipe of verl~\citep{sheng2024hybridflow} with disaggregated rollout and training: generation runs on vLLM~\citep{kwon2023vllm} workers and optimization on FSDP2 workers, on a single node of eight A100 GPUs per run, four for each.
The rollout workers generate from a weight snapshot, the learner takes one optimizer step per incoming batch, and the weights are synchronized every $S$ learner steps, which realizes the lag sequence in \eqref{eq:phi-t}.
No stabilization mechanism is enabled: there is no off-policy correction beyond the GRPO clip, no gradient projection, and no staleness-aware learning-rate control.

\paragraph{Configuration.}
Table~\ref{tab:hparams} lists the settings.
Within a model family, all $(S,\eta)$ configurations share the data order, initialization, and remaining hyperparameters, so that $S$ and $\eta$ are the only varying factors.
Each configuration is run with seed $42$; Appendix~\ref{app:multiseed} repeats the $S=32$ runs of the 3B policy with seeds $100$ and $2026$.

\begin{table}[htbp]
\centering
\caption{Training and evaluation configuration, shared by all runs unless stated otherwise. Values marked $^\dagger$ are framework defaults that the launch scripts do not override.}
\label{tab:hparams}
\small
\begin{tabular}{ll}
\toprule
Setting & Value\\
\midrule
Training data & MATH train, 7.5k (Llama); DAPO-Math-17k (Qwen3)\\
Validation data & MATH test, 5k (Llama); MATH-500, AMC 23/24, Minerva (Qwen3)\\
Reward & binary exact match, $\{0,1\}$\\
\midrule
Algorithm & vanilla GRPO\\
Group size $G$ / prompts per step & $8$ / $256$ (one update per rollout batch)\\
Clip range $\epsilon$ / loss aggregation & $0.2^\dagger$ / token-mean$^\dagger$\\
KL term & in the loss, $\beta=10^{-3}$, k3 estimator\\
Entropy coefficient & $0$\\
\midrule
Staleness $S$ (sync interval) & $\{8,16,32\}$\\
Learning rate $\eta$ & $\{10^{-6},\,5\times10^{-7},\,2\times10^{-7},\,10^{-7}\}$, constant, no warmup\\
Optimizer & AdamW, $(\beta_1,\beta_2)=(0.9,0.999)^\dagger$, weight decay $0.01^\dagger$\\
Gradient-norm clipping & $1.0^\dagger$\\
Micro-batch per device & $4$\\
\midrule
Max prompt / response length & $512$ / $4096$ tokens\\
Rollout temperature / top-$p$ & $1.0^\dagger$ / $1.0^\dagger$\\
Evaluation & every $25$ steps, temperature $1.0$, top-$p$ $1.0$\\
Seeds & $42$; also $100$, $2026$ for the runs of Appendix~\ref{app:multiseed}\\
\bottomrule
\end{tabular}
\end{table}

\paragraph{Stopping criterion.}
Runs have no fixed step budget.
A run is stopped once collapse is identified
(Section~\ref{subsec:prelim-async}), or once it has clearly
outlived the horizon $c_2/\eta$ at which collapsing runs with
the same learning rate fail; the curves therefore end at
different steps.

\paragraph{Evaluation protocol.}
The Llama policies are evaluated on the $5$k MATH test split, which is disjoint from the $7.5$k training split and never used for rollouts, with the same rule-based checker as in training.
The Qwen3 policies are evaluated on MATH-500~\citep{lightman2024verify}, AMC 2023 and AMC 2024~\citep{aops_amc}, and Minerva Math~\citep{lewkowycz2022minerva}, reporting avg@1 on MATH-500 and Minerva Math and avg@8 on the two AMC sets.
Decoding settings and the number of samples per problem are fixed across configurations and evaluation steps.

\paragraph{Gradient cosine.}
The cosine $c_t$ is taken between the raw gradients of consecutive learner steps, before the optimizer state is applied, following the protocol of \citet{xu2026gac}: the required inner products and norms are accumulated across data-parallel ranks instead of materializing the full gradient vectors.
\providecommand{\tbd}[1]{\textcolor{red}{[#1]}}

\section{Robustness to the Random Seed}
\label{app:multiseed}

The sweeps of Section~\ref{sec:experiments} use one seed per configuration.
To test whether the regularities depend on the seed, we repeat the $S=32$ runs of \texttt{Llama-3.2-3B-Instruct} with seeds $42$, $100$ and $2026$.
Nothing else changes: data, pipeline, and hyperparameters are those of Section~\ref{sec:experiments} and Appendix~\ref{app:implementation}.
We use $S=32$ because it is the staleness at which the sweep samples both sides of the stability boundary: the four learning rates of the main grid collapse (Figure~\ref{fig:seed_highlr}), and the four extended-horizon learning rates are stable (Figure~\ref{fig:seed_smalllr}).

\paragraph{Collapsing runs.}
Table~\ref{tab:seed} reports $M_{\mathrm{collapse}}$ for each run,
using the abrupt reward drop defined in
Section~\ref{subsec:prelim-async}.
At every learning rate, the three seeds collapse in the same rollout batch.
The per-batch mean training rewards before the collapse agree to within $0.02$, and even the two-stage drop at $\eta=2\times10^{-7}$, a partial drop one batch before the collapse, is reproduced by all seeds.
The seed therefore moves $t_{\mathrm{obs}}$ by less than the observation granularity $S$, so, for the 3B policy at $S=32$, the clock of Section~\ref{subsec:escape} is a property of $(S,\eta)$ and not of a particular sample path.
The gradient cosine is equally reproducible:
at $\eta=10^{-7}$, in every seed $c_t$ decays to about $0.05$
after the initial reward climb and rises again shortly before
the collapse, as described in Section~\ref{subsec:readout-exp}.

\begin{table}[h]
\centering
\caption{First collapsed rollout batch $M_{\mathrm{collapse}}$ of \texttt{Llama-3.2-3B-Instruct} at $S=32$ under three seeds, with $t_{\mathrm{obs}}=M_{\mathrm{collapse}}S$. All seeds agree in every column.}
\label{tab:seed}
\small
\begin{tabular}{lcccc}
\toprule
Seed & $\eta=10^{-6}$ & $5\times10^{-7}$ & $2\times10^{-7}$ & $10^{-7}$\\
\midrule
$42$   & $1$ & $2$ & $5$ & $10$\\
$100$  & $1$ & $2$ & $5$ & $10$\\
$2026$ & $1$ & $2$ & $5$ & $10$\\
\midrule
$t_{\mathrm{obs}}$ & $32$ & $64$ & $160$ & $320$\\
\bottomrule
\end{tabular}
\end{table}

\paragraph{Stable runs.}
At $\eta\in\{8,7,6,5\}\times10^{-8}$, all three seeds train stably for $500$ steps, and their reward, gradient-cosine, and validation curves are nearly indistinguishable.
At $8\times10^{-8}$ and $7\times10^{-8}$ this horizon exceeds the clock $c_2/\eta$ of $400$ and $457$ steps, so every seed survives past the point at which a run above the boundary would have collapsed; the stable side of the bracket at $S\eta=2.56\times10^{-6}$ is therefore reproduced across seeds.
At $6\times10^{-8}$ and $5\times10^{-8}$ the $500$-step horizon is shorter than the clock, and these runs show seed agreement only; their long-horizon behavior is covered by the single-seed runs of Figure~\ref{fig:smalllr_3b}.

\paragraph{Why the seeds agree.}
The seed changes the data order and the rollout samples, which enter the update through the zero-mean noise $\xi_t$.
Above the boundary, the collapse time is set by the coherent part of the drift, and the noise perturbs it only at relative order $O(\tau^{-1/2})$ (Theorem~\ref{thm:collapse}); below the boundary, the noise-driven displacement is predicted to stay far from the edge of $\mathcal{N}_R$ over the horizons we run (Corollary~\ref{cor:incoherent}).
In both regimes the seed should therefore leave the observable outcome unchanged, which is what the runs show.

\begin{figure}[htbp]
    \centering
    \includegraphics[width=\linewidth]{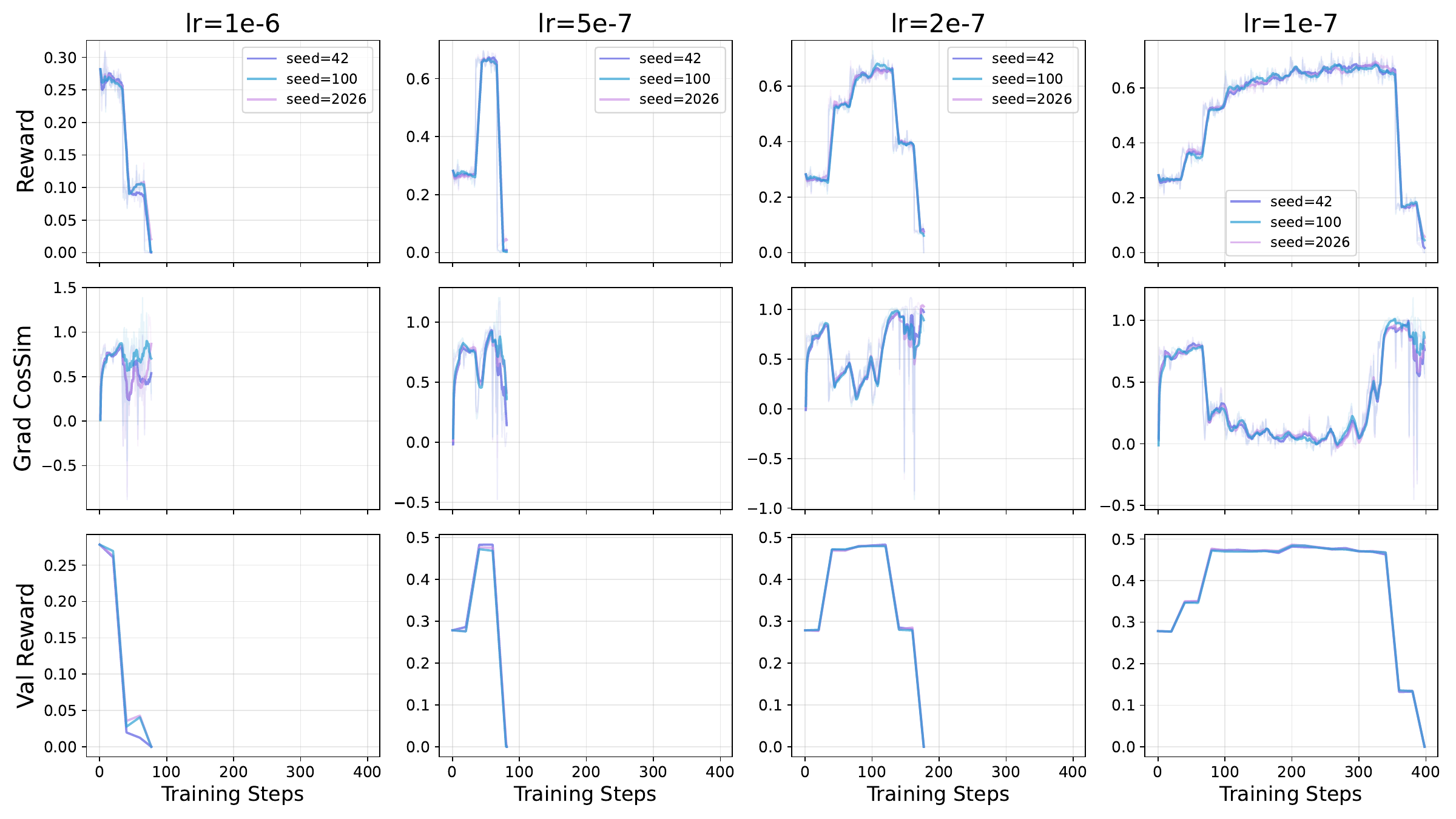}
    \caption{Seed robustness above the boundary: \texttt{Llama-3.2-3B-Instruct} at $S=32$ with the four learning rates of the main grid, three seeds each.
    \emph{Top:} training reward; \emph{middle:} gradient cosine $c_t$; \emph{bottom:} validation reward.
    All seeds collapse in the same rollout batch (Table~\ref{tab:seed}), with nearly identical curves before the collapse.}
    \label{fig:seed_highlr}
\end{figure}

\begin{figure}[htbp]
    \centering
    \includegraphics[width=\linewidth]{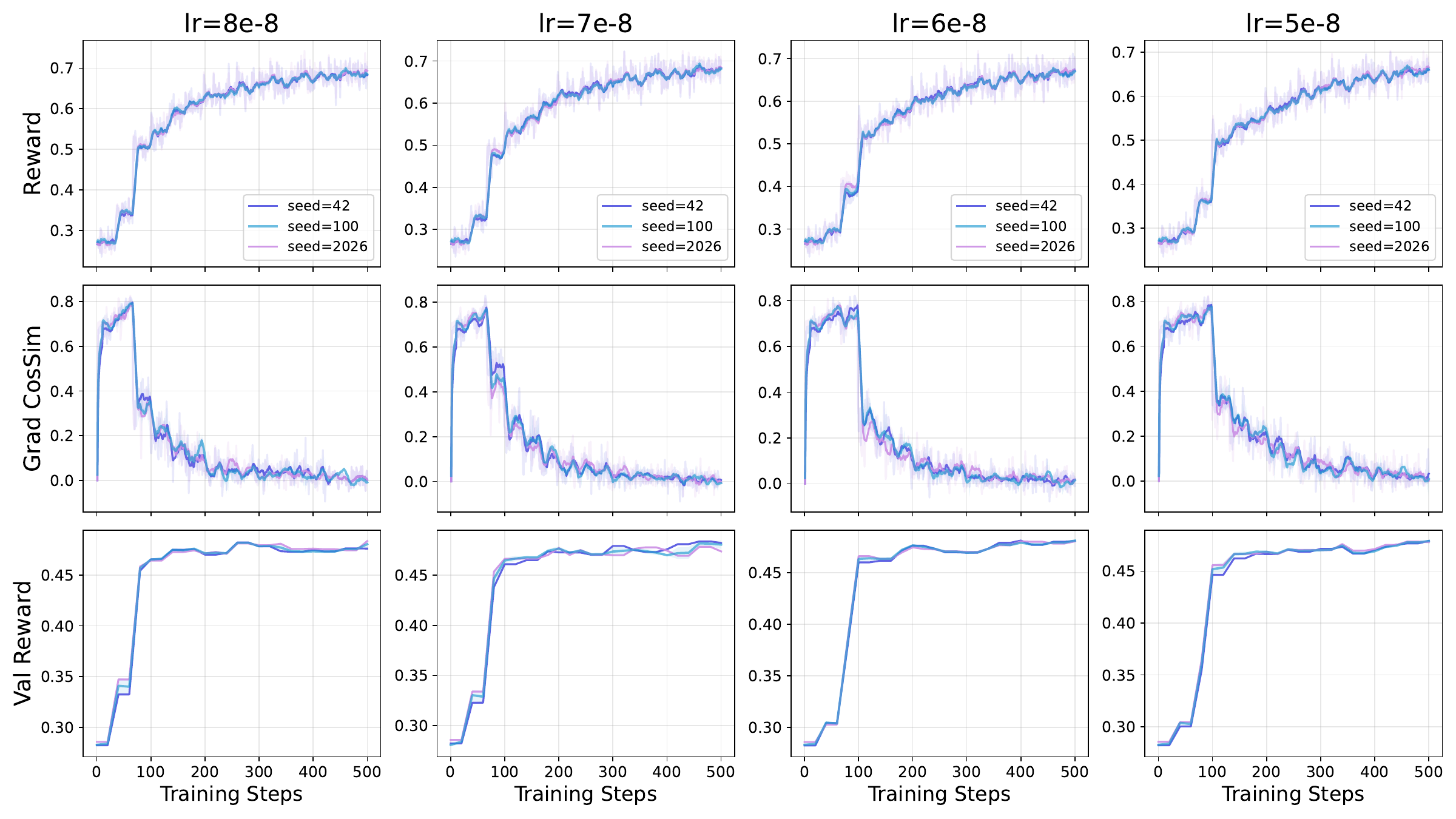}
    \caption{Seed robustness below the boundary: \texttt{Llama-3.2-3B-Instruct} at $S=32$ with the four extended-horizon learning rates, three seeds each, over $500$ steps.
    All runs are stable, and $c_t$ decays to zero after the initial reward climb in every seed.}
    \label{fig:seed_smalllr}
\end{figure}

\section{Extended Related Work}
\label{app:related}

This appendix expands Section~\ref{sec:related-work}.
Our work sits between three lines of research that rarely meet: systems that make LLM post-training asynchronous and must live with stale rollouts, algorithms and analyses that make policy optimization tolerate off-policy data, and the optimization literature on delayed updates and adaptive optimizers.
We review each in turn and, at the end of every paragraph, state how our analysis differs.
In short, prior work either chooses the lag empirically, corrects it inside the objective, or guarantees stability under sufficient conditions; we keep the objective fixed and characterize which $(S,\eta)$ configurations fail, and when.

\paragraph{Asynchronous RL for LLMs: how staleness is handled.}
Because generation dominates the cost of RL for LLMs, recent systems decouple it from optimization~\citep{sheng2024hybridflow,hu2024openrlhf}, through asynchronous actor--learner pipelines~\citep{noukhovitch2025asynchronous,fu2025areal,wu2025llamarl}, disaggregated and streaming generation~\citep{zhong2025streamrl}, in-flight weight updates~\citep{piche2025pipelinerl}, or partial rollouts~\citep{kimi2025k15}.
ECHO-2~\citep{song2026echo2} distributes rollouts over wide-area inference workers and makes bounded staleness an explicit, user-controlled parameter, trading it against dissemination latency and cost.
Such systems handle the resulting staleness in one of two ways.
They either cap the lag at a hand-tuned value~\citep{fu2025areal,khatri2025art,song2026echo2} or restore on-policy updates at every iteration~\citep{lu2025periodic}, or they modify the objective so that stale data can be used, through decoupled proximal policies~\citep{hilton2022batch,fu2025areal,li2025a3po}, tapered importance weights~\citep{leroux2025topr}, second-moment trust regions~\citep{zheng2025prosperity}, relaxed clipping~\citep{tian2026mugrpo}, or projection of stale-aligned gradients~\citep{xu2026gac}.
\emph{Difference:} we keep vanilla GRPO without additional staleness
correction and characterize an approximately constant-$S\eta$ stability
boundary in the Llama sweeps, together with a collapse clock in
$\sum_t\eta_t$ under the stated drift conditions.
These findings provide a baseline against which such corrections can be
measured, while our conditional readout model explains how the
consecutive-gradient cosine used by \citet{xu2026gac} can provide only
a late warning near the boundary.

\paragraph{Off-policy corrections in distributed RL.}
Decoupling acting from learning is the central design pattern of distributed deep RL, from A3C~\citep{mnih2016asynchronous} to IMPALA~\citep{espeholt2018impala} and SEED RL~\citep{espeholt2020seed}, and the resulting policy lag is controlled with truncated importance weights such as Retrace and V-trace~\citep{munos2016safe,espeholt2018impala} or with sample-reuse variants of PPO~\citep{queeney2021generalized}.
Other systems avoid the lag altogether by remaining synchronous~\citep{wijmans2020ddppo}, and reproducibility studies show how stale rollouts perturb actor--learner PPO dynamics~\citep{huang2023cleanba}.
\emph{Difference:} our vanilla GRPO setup uses the clipped ratio and a
KL penalty without additional staleness correction; we characterize its
stability as a function of $(S,\eta)$ with other settings fixed.

\paragraph{Analyses of delayed and off-policy updates: what is guaranteed.}
Asynchronous SGD is analyzed as gradient descent on a fixed objective with gradients evaluated at delayed iterates, and convergence holds when the step size shrinks with the maximum delay~\citep{recht2011hogwild,lian2015asynchronous,zhang2016staleness,dutta2018slow,koloskova2022sharper,mishchenko2022asynchronous}; asynchrony has also been shown to act as an implicit momentum~\citep{mitliagkas2016asynchrony}.
For GRPO, recent analyses show that an update improves the objective when the behavior policy stays close to the learner~\citep{mroueh2026revisiting}, and that group-relative REINFORCE admits an off-policy interpretation in which regularizing the update is what tolerates the mismatch~\citep{yao2026grouprelative}.
\emph{Difference:} these are sufficient conditions, for convergence or
for a single improvement step.
We complement these guarantees by examining collapse in vanilla GRPO,
where delay also shifts the sampling distribution and importance ratios.
The Llama sweeps exhibit an approximately constant-$S\eta$ stability
boundary and inverse-learning-rate collapse timing that is insensitive
to $S$ across the tested unstable configurations.

\paragraph{Optimizer dynamics: interpreting the collapse clock.}
Adam can exhibit sign-like behavior through coordinatewise
normalization~\citep{bernstein2018signsgd,balles2018dissecting,
kunstner2023noise}, and raw gradient scales in transformers differ
widely across parameter blocks~\citep{zhang2024transformers}.
Learning-rate thresholds such as the edge of
stability~\citep{cohen2021gradient}, scaling rules for adaptive
methods~\citep{malladi2022sdes}, and learning-rate sensitivity studies
of large transformers~\citep{wortsman2024small} tie stability to
optimizer settings.
\emph{Difference:} our conditional model uses these properties to explain
two staleness-specific effects: when persistent drift has comparable
saturated speed and exit geometry across $S$, the leading collapse budget
in $\sum_t\eta_t$ is approximately shared; heterogeneous gradient scales
can mask coherent normalized updates in the raw-gradient cosine under
the stated readout conditions.

\section{Additional Training Curves}
\label{app:extra-figures}

Figure~\ref{fig:metrics_grid_1b} shows the 1B sweep summarized in Table~\ref{tab:collapse}; the remaining figures split the training reward of both Llama sweeps by staleness.

\begin{figure}[htbp]
    \centering
    \includegraphics[width=0.99\linewidth]{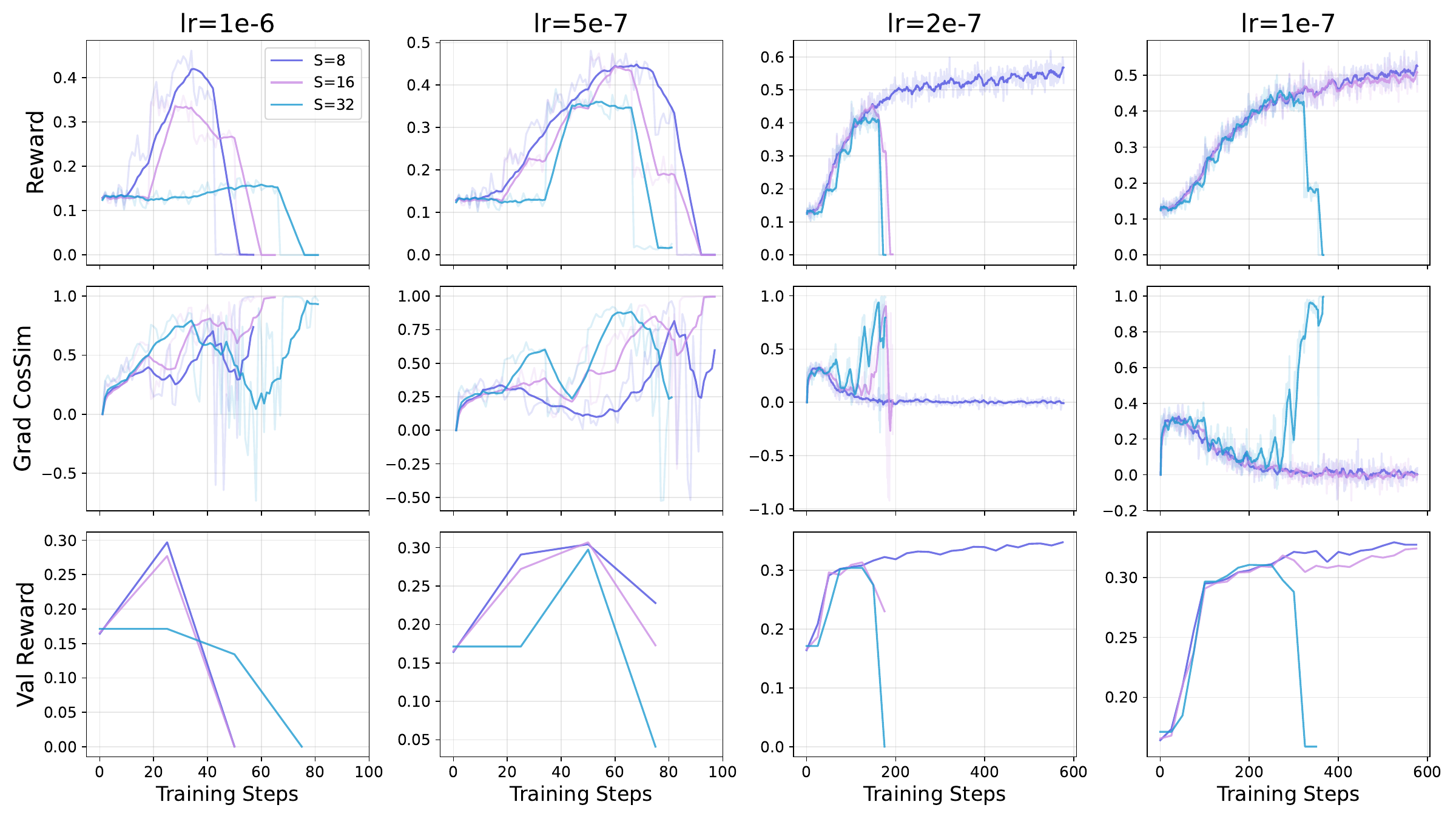}
    \caption{Staleness--learning-rate sweep on \texttt{Llama-3.2-1B-Instruct}, with the layout of Figure~\ref{fig:metrics_grid_3b}.
    The collapse counts coincide with those of the 3B policy (Table~\ref{tab:collapse}), so both sizes share the boundary in $S\eta$ and the clock constant $c_2=3.2\times10^{-5}$.}
    \label{fig:metrics_grid_1b}
\end{figure}
\begin{figure}[htbp]
    \centering
    \includegraphics[width=\linewidth]{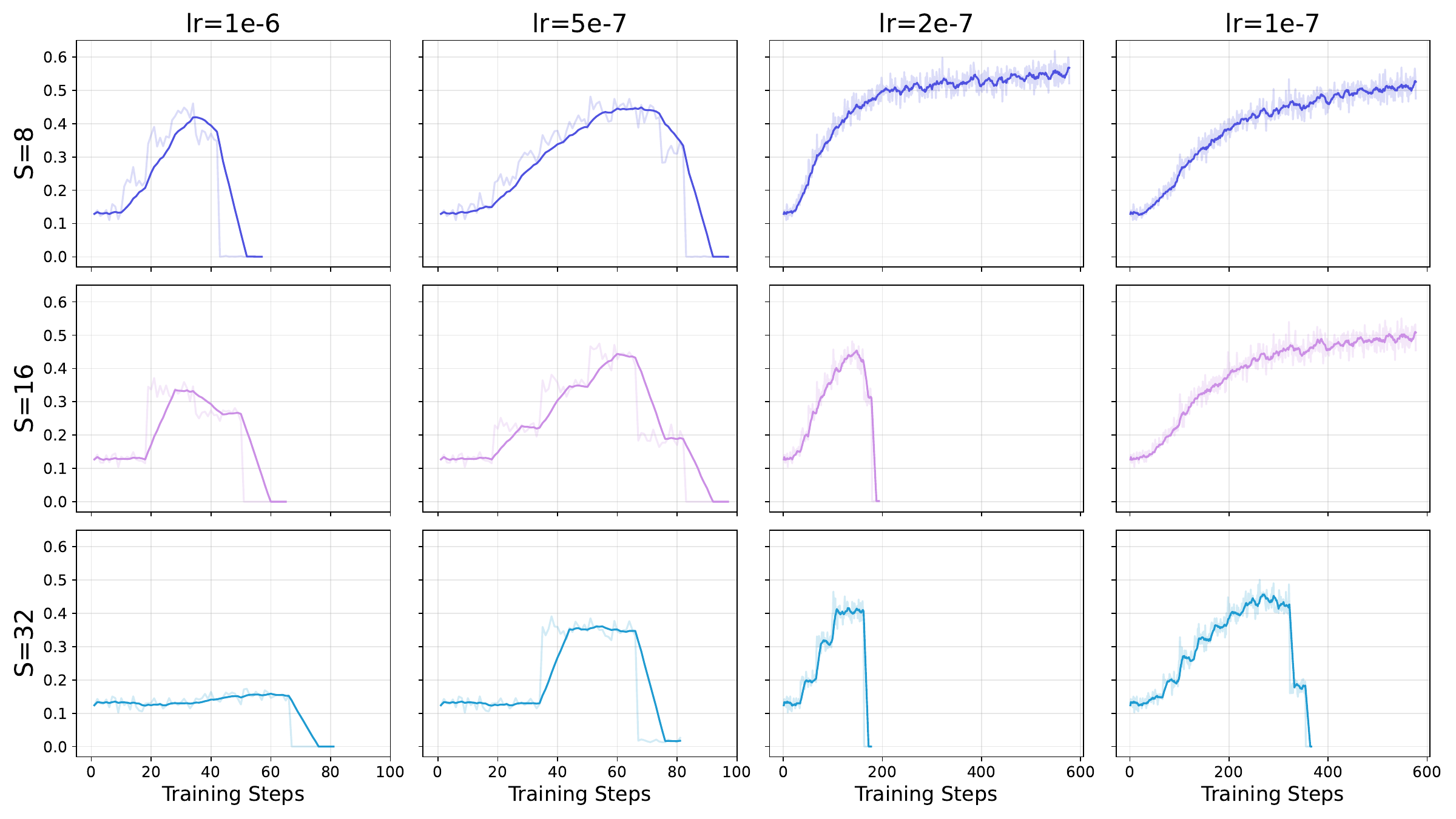}
    \caption{Per-staleness training-reward curves on
    \texttt{Llama-3.2-1B-Instruct}, separated into individual panels
    for legibility (companion to Figure~\ref{fig:metrics_grid_1b}).
    Each panel fixes one $(S,\eta)$ pairing; collapse corresponds
    to an abrupt, severe drop in training reward.
    The onset of collapse is consistent with the collapse-time scaling
    $t_{\mathrm{obs}}\eta\!\approx\!c_2$ reported in
    Table~\ref{tab:collapse}.}
    \label{fig:reward_grid_1b}
\end{figure}

\begin{figure}[htbp]
    \centering
    \includegraphics[width=\linewidth]{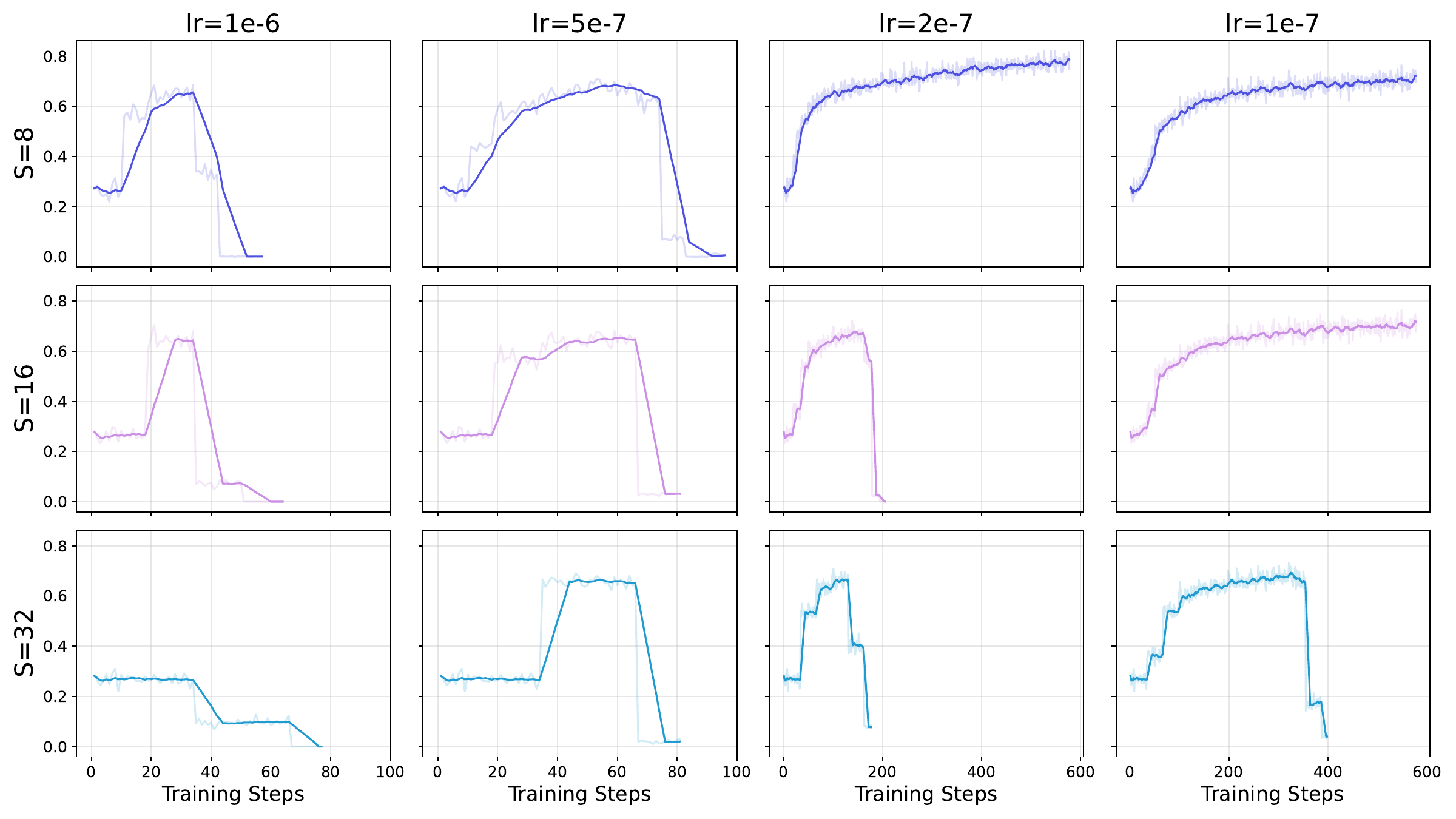}
    \caption{Per-staleness training-reward curves on
    \texttt{Llama-3.2-3B-Instruct}, separated into individual panels
    for legibility (companion to Figure~\ref{fig:metrics_grid_3b}).
    As with the 1B policy, the collapse onset across $(S,\eta)$
    is consistent with the approximately $S$-independent
    collapse-time scaling $t_{\mathrm{obs}}\eta\!\approx\!c_2$
    (Section~\ref{subsec:escape}).}
    \label{fig:reward_grid_3b}
\end{figure}

\end{document}